\documentclass{article}

\PassOptionsToPackage{numbers, compress}{natbib}

\usepackage[preprint]{neurips_2023}

\usepackage[utf8]{inputenc} 
\usepackage[T1]{fontenc}    
\usepackage{hyperref}       
\usepackage{url}            
\usepackage{booktabs}       
\usepackage{amsfonts}       
\usepackage{nicefrac}       
\usepackage{microtype}      
\usepackage{xcolor}         
\usepackage{mystyle}

 \title{Multitask Learning with No Regret:\\ from Improved Confidence Bounds to Active Learning}

\author{%
  Pier Giuseppe Sessa\thanks{equal contribution} \\
  ETH Z\"urich\\\texttt{piergiuseppe.sessa@inf.ethz.ch}
    \And
    Pierre Laforgue$^*$ \\
  Università degli Studi di Milano
\\\texttt{pierre.laforgue@unimi.it}
  \newline
  \And 
  Nicolò Cesa-Bianchi\\
    Università degli Studi di Milano
\\
  \texttt{nicolo.cesa-bianchi@unimi.it}
  \And
  Andreas Krause \\
  ETH Z\"urich\\
  \texttt{krausea@ethz.ch}
}

\begin{document}

\maketitle

\begin{abstract}
Multitask learning is a powerful framework that enables one to simultaneously learn multiple related tasks by sharing information between them.
Quantifying uncertainty in the estimated tasks is of pivotal importance for many downstream applications, such as online or active learning.
In this work, we provide novel multitask confidence intervals in the challenging agnostic setting, i.e., when neither the similarity between tasks nor the tasks' features are available to the learner.
The obtained intervals do not require i.i.d.\ data and can be directly applied to bound the regret in online learning.
Through a refined analysis of the multitask information gain, we obtain new regret guarantees that, depending on a task similarity parameter, can significantly improve over treating tasks independently.
We further propose a novel online learning algorithm that achieves such improved regret without knowing this parameter in advance, i.e., automatically adapting to task similarity.
As a second key application of our results, we introduce a novel multitask active learning setup where several tasks must be simultaneously optimized, but only one of them can be queried for feedback by the learner at each round.
For this problem, we design a no-regret algorithm that uses our confidence intervals to decide which task should be queried.
Finally, we empirically validate our bounds and algorithms on synthetic and real-world (drug discovery) data.\looseness=-1
\end{abstract}

\section{Introduction}

In many real-world applications, one often faces multiple related tasks to be solved sequentially or simultaneously. 
The goal of multitask learning (MTL)~\cite{caruana1997multitask} is to leverage the similarities across the tasks to obtain more accurate and robust models. Indeed, by jointly learning multiple tasks, MTL can exploit their statistical dependencies, yielding better generalization and faster learning than treating each task independently. MTL has gained significant attention in recent years, as it has been shown to be effective in a wide range of applications, including natural language processing, computer vision, federated learning, and drug discovery, see e.g.,~\cite{collobert2008unified, liu2019end,kairouz2021advances,widmer2010inferring}.

A very natural model for learning across multiple tasks is the agnostic multitask (MT) regression approach of~\cite{evgeniou2005learning}. This utilizes a multitask kernel that can interpolate between running $N$ (number of tasks) independent regressions, and regressing all tasks to their common average, depending on a tunable parameter. Notably, such a kernel does not require any knowledge neither about tasks' features nor about their similarity, thus finding good application in several domains. For instance, \citet{cavallanti2010linear} study it for online classification, and \citet{cesa2013gang} for online convex optimization.\looseness-1

However, it is much less understood how to \emph{quantify the uncertainty} of such MT regression, i.e., assessing confidence in the estimated tasks. In particular, as also outlined by~\cite{evgeniou2005learning} as an open problem, it is important to assess their generalization error as a function of the kernel parameter. Appropriately characterizing these confidence intervals is indeed of crucial importance for a whole set of downstream applications. More concretely, multitask confidence intervals are used in online learning to inform the next decision to be made~\cite{cesa2013gang}. In active learning---as we show next---these intervals are pivotal to deciding the most informative task to query. 


In this work, we study the agnostic MT regression setup of~\cite{evgeniou2005learning}, and provide \emph{new multitask confidence intervals} (see Figure~\ref{fig:new_intervals} for a visualization) for the full range of the kernel parameter. Our intervals hold in the so-called adaptive setting, i.e., without requiring i.i.d.~data, and are \emph{tighter up to a $\sqrt{N}$ factor} than the naive ones employed in~\cite{cesa2013gang}. Moreover, we provide the first bounds for the information gain of MT regression and utilize them---together with the derived intervals---to obtain \emph{tighter online learning guarantees}.
The latter depend on a task similarity parameter and can significantly improve over treating tasks independently. Additionally, we propose an adaptive no-regret algorithm that exploits task similarity without knowing this parameter in advance. Finally, we consider a novel multitask \emph{active learning} setup, where tasks should be simultaneously optimized but only one of them can be queried at each round. We show that the newly derived intervals are also crucial in such a setting, and provide a new algorithm that ensures sublinear regret. We demonstrate the superiority of the derived intervals over previously proposed algorithms on synthetic as well as real-world drug discovery tasks.\looseness-1

\begin{figure}[t]
    \includegraphics[width=.89\textwidth]{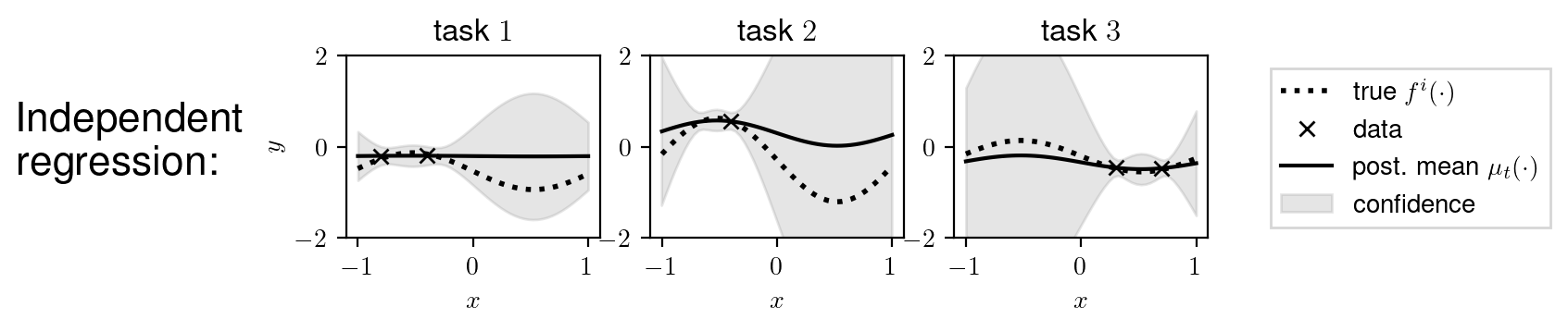}

    \vspace{-1em}
    \includegraphics[width=.99\textwidth]{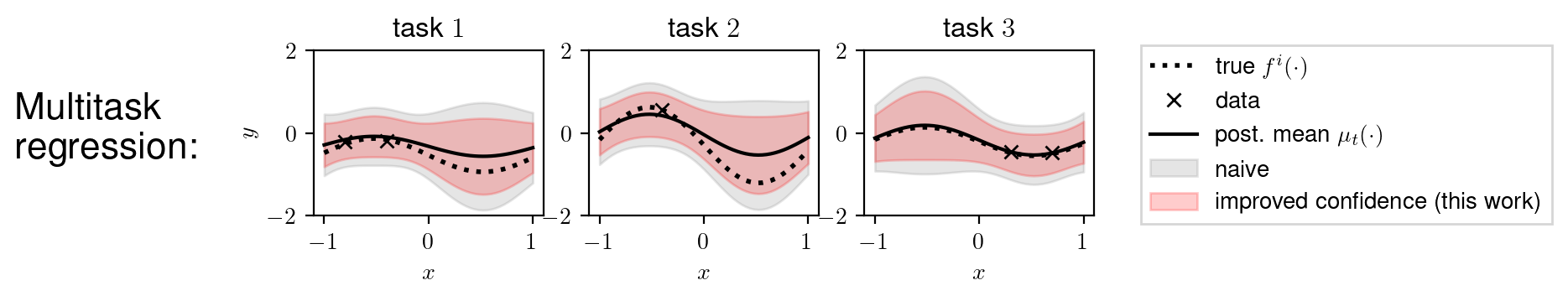}
    \vspace{-1em}
    \caption{Independent vs. Multitask (MT) regression. MT regression leverages data coming from multiple related tasks and can yield more accurate and more confident estimates. In this work, we show naive confidence intervals are overly conservative and provide improved ones (shaded in red).}
    \label{fig:new_intervals}
    \vspace{-0.5cm}
\end{figure}

\textbf{Related work.} The agnostic MT regression approach of~\cite{evgeniou2005learning} reduces the learning of $N$ tasks to a single regression problem, as a function of the MT kernel parameter. When combined with support vector machines, it was shown effective in a series of classification problems~\cite{evgeniou2005learning, sheldon2008graphical}, and since then was studied in various further settings. \citet{cavallanti2010linear}, e.g., analyze mistake bounds for online MT classification algorithms as a function of the kernel parameter. \citet{cesa2013gang}, instead, utilize the MT kernel to prove regret bounds in online MT learning with bandit feedback. Inspired by this, \citep{ciliberto2015learning} focuses on learning more general kernel structures from data. An important question not addressed by previous work, though, is how to properly quantify the uncertainty of the obtained task estimates. This problem is well-understood in single-task learning (e.g.,~\cite{abbasi2011improved, srinivas2010gaussian, chowdhury2017}) but remains largely unexplored in MT domains. As shown in \citep{cesa2013gang}, MT confidence intervals can in principle be obtained by a naive application of the single-task guarantees of~\cite{abbasi2011improved}. However, as we show in~\Cref{sec:mt_setting}, the so-obtained intervals are extremely conservative and---as a result---can hamper the MT learning performance. Our intervals are tighter by a factor up to $\sqrt{N}$ w.r.t. the naive ones from~\cite{cesa2013gang}, yielding novel online learning regret guarantees which can provably improve over treating tasks independently.\looseness-1

Compared to MT online learning~\cite{cavallanti2010linear, cesa2013gang}, where a single task is revealed to the learner at each round, a series of works have considered learning multiple tasks \emph{simultaneously}, i.e., taking a decision for each one of them. \citet{dekel2007online}, e.g., propose the use of a shared loss function to account for tasks' relatedness, \citet{lugosi2009} studies the computational tractability of taking multiple actions with joint constraints, while \citet{cavallanti2010linear} propose a matrix-based extension of the multitask Perceptron algorithm.
In all of these works, however, the learner receives feedback from \emph{all} the tasks. In~\Cref{sec:active_learning}, instead, we focus on the challenging setup where only one task can be queried by the learner at each round. Hence, in addition to choosing good actions, the learner faces the \emph{active learning} task of assessing the most informative feedback, in order to achieve sublinear regret. Perhaps most related to ours, is the offline contextual Bayesian optimization setup of~\cite{char2019offline,li2022near} where the goal is to compute the best strategy for each context (task) with minimal function interactions. However, unlike us, \cite{char2019offline,li2022near} do not guarantee sublinear regret but provide only sample-complexity results.

Finally, we note that MT confidence intervals and regret guarantees were also recently derived by~\citep{chowdhury2021no}, albeit in a different setup and regression model. Indeed, the authors of~\citep{chowdhury2021no} focus on multi-objective optimization where they ought to learn \emph{multi-output} functions (each output corresponding to a task) using matrix-valued kernels. Although their setup can be related to ours, it crucially requires all tasks to be observed at each round, leading to different challenges than ours, see \Cref{app:intuition_regression} for details.

\textbf{Notation.} We use $[N] \coloneqq \{1, \ldots, N\}$, $\mathbbm{1}_N$ for the vector in $\mathbb{R}^N$ full of ones. Norms of functions are always taken w.r.t. the natural RKHS norm, so that we drop the subscript for simplicity of writing.

\section{Improved Confidence Intervals for Multitask Kernel Regression}
\label{sec:mt_setting}

In this section, we introduce the MT kernel regression setting, and prove our refined confidence intervals.
Of independent interest, these results are then leveraged in \Cref{sec:online,sec:active_learning} to derive novel regret bounds for online and active multitask learning.
All proofs are deferred to the Appendix.


\subsection{Multitask Kernel Regression}
\label{sec:mt_regression}

Given an input space $\X$, equipped with a (single task) scalar kernel $k_\X \colon \X \times \X \rightarrow \mathbb{R}$, the goal
of MT kernel regression is to jointly learn $N$ different functions $f_1, \ldots, f_N$ from $\X$ to $\mathbb{R}$, all belonging to $\H_{k_\X}$, the RKHS associated to $k_\X$.
To do so, the learner is given a set of triplets $\{(i_s, x_s), y_s\}_{s=1}^t$ consisting of a measured task index $i_s \in [N]$, a measured point $x_s \in \X$, and a noisy measurement $y_s = f_{i_s}(x_s) + \xi_s$, where $\xi_s$ is an independent random variable to be specified later.
We can further define the multitask function $\fmt \colon [N] \times \X \rightarrow \mathbb{R}$ such that $\fmt(i, \cdot) = f_i$, and the multitask kernel
\begin{equation}\label{eq:mt_kernel}
k\big((i,x), (i', x')\big) =  k_\T(i, i') \cdot k_\X(x, x')\,,
\end{equation}
where $k_\T$ is a kernel on the tasks.
In certain cases, the latter might be given as input to the learner, either under the form of the task Gram matrix, or via task features and an assumed (e.g., linear) similarity~\cite{krause2011contextual}.
However, in practice such information is usually not accessible to the learner.
In such a case, a standard \emph{agnostic} approach to MT regression~\citep{cavallanti2010linear,cesa2013gang,ciliberto2015learning,evgeniou2005learning,sheldon2008graphical} then consists in leveraging a parameterized task kernel of the form
\begin{equation}\label{eq:A_clique}
k_\T(i, i') = \big[K_\text{task}(b)\big]_{ii'}\,, \quad  \text{with} \quad K_\text{task}(b) = \frac{1}{1+b} I_N + \frac{b}{1+b} \frac{\mathbbm{1}_N\mathbbm{1}_N^\top}{N} \in \mathbb{R}^{N \times N}\,.
\end{equation}
Intuitively, parameter $b \ge 0$ governs how similar the tasks are thought to be.
When $b = 0$, we have $K_\text{task}(b) = I_N$, such that $k_\T(i,i') = \delta_{ii'}$, and the tasks are considered to be independent.
When $b$ goes to $+\infty$, we have $K_\text{task}(b) = \mathbbm{1}_N\mathbbm{1}_N^\top/N$, and all tasks are considered to be one single common task.
Any choice of $b \in (0, +\infty)$ corresponds to a tradeoff between these two regimes.
We make this intuition explicit in \Cref{prop:extreme_cases} (\Cref{app:intuition_regression}).
Note that all quantities depending on the kernel do by definition depend on $b$.
We use the notation $\vert\, b$ to make this dependence explicit when relevant.

Given a history of measurements $\{(i_s, x_s), y_s\}_{s=1}^t$, one may then estimate $\fmt$, or equivalently the $\{f_i\}_{i=1}^N$, by standard kernel Ridge regression using the MT kernel $k$.
One obtains the estimates
\begin{align}
\mu_{t}(i, x\,\vert\,b) &= \bm{k}_t(i,x)^\top\big(K_t + \lambda I_t \big)^{-1} \bm{y}_{1:t}\,, \label{eq:kernel_regression_mean}\\
\sigma_{t}^2(i, x\,\vert\,b) &= k\big((i,x),(i,x)\big) -\bm{k}_t(i,x)^\top \big(K_t + \lambda I_t \big)^{-1} \bm{k}_t(i,x) \,,\label{eq:kernel_regression_var}
\end{align}
where $\bm{k}_t(i,x) = \big[k\big((i_s, x_s), (i,x)\big)\big]_{s=1}^t$, $K_t = \big[k \big((i_s, x_s), (i_{s'}, x_{s'})\big)\big]_{s,s'=1}^t$, $\bm{y}_{1:t}= \big[y_s \big]_{s=1}^t$, and $\lambda > 0$ is some regularization parameter. Functions $\mu_t$ and $\sigma_t^2$ can be interpreted as the posterior mean and variance of a corresponding Gaussian Process model, see \citep{srinivas2010gaussian,chowdhury2017}. In the next section, we will utilize $\mu_{t}$ and $\sigma_{t}^2$ to construct high-probability confidence intervals for the multitask function $\fmt$.

\paragraph{Information gain.}
An important quantity when analyzing (multitask) kernel regression is the so-called \emph{(multitask) information gain}: 
\begin{equation*}
\gamma_T^\text{mt}(b) = \frac{1}{2} \ln \big| I_T + \lambda^{-1}K_T \big|\,.
\end{equation*}
It can be interpreted as the reduction in uncertainty about $\fmt$ after having observed a given set of $T$ datapoints.
Similarly to single-task setups \cite{srinivas2010gaussian,chowdhury2017}, we use $\gamma^\text{mt}_T$ in the next sections to characterize our confidence intervals and regret bounds.
Note that $\gamma^\text{mt}_T$ depends on the multitask kernel through $K_T$, and hence on $b$.
In \Cref{sec:online}, we exploit the properties of our multitask kernel to obtain a sharper control over $\gamma^\text{mt}_T$, which is then fundamental to derive improved regret bounds.


\subsection{Improved Confidence Intervals}

In this section, we utilize the regression estimates obtained in \Cref{eq:kernel_regression_mean,eq:kernel_regression_var} to construct high probability confidence intervals around the unknown multitask function $\fmt$.
First, we assume that $\|f_i\| \leq B$ for all $i \in \uptoN$, as it is standard in single-task regression.
Moreover, let $f_\text{avg} = (1/N)\sum_{i=1}^N f_i~$ be the average task function, and define
%
\begin{equation}\label{eq:def_eps}
   \epsilon = \max_i \, \|f_i - f_\text{avg}\|/B\,.
\end{equation}

   \vspace{-0.3cm}
Note that by definition $\epsilon \in [0, 2]$.
Quantity $\epsilon$ measures how much individual tasks deviate from the average task $f_\text{avg}$.
The smaller $\epsilon$, the more similar the tasks are, the limit case being that all tasks are equal, attained at $\epsilon=0$.
At the other extreme, when $\epsilon \gg 0$ tasks are highly distant and ought to be learned independently.
The deviation $\epsilon$ plays a crucial role in the subsequent analysis.

\textbf{A naive confidence interval.}
As discussed in \cite{cesa2013gang}, it is possible to construct the multitask feature map $\tpsi$ associated to $k$.
One may then rewrite $\fmt(i, x) = \langle \tilde{f}, \tpsi(i, x)\rangle$, where $\tilde{f}$ is a transformed version of $\fmt$ which satisfies $\big\|\tilde{f}\big\| \le B\sqrt{N(1 + b\epsilon^2)}$, see \Cref{app:intuition_regression,apx:interval} for details.
MT regression thus boils down to single-task regression, over the modified features $\tilde{\psi}(i,x)$, and with target function~$\tilde{f}$.
One can then employ well-known linear regression results to obtain confidence intervals for $\fmt$.
Using \citep[Theorem~3.11, Remark~3.13]{abbasi2012online} and the definition of $\gamma_t^\text{mt}(b)$, with probability $1-\delta$ we have that for all $t \in \mathbb{N}$, $i \in \uptoN$, and $x \in \X$ it holds $\big|\mu_t(i,x\,|\,b) - \fmt(i,x)\big| \le \beta_t^\text{naive}(b) \cdot \sigma_t(i,x\,|\,b)$, where\looseness-1
\[
\beta_t^\text{naive}(b)  = B\sqrt{N(1 + b\epsilon^2)} + \lambda^{-1/2}\sqrt{2\big(\gamma_t^\text{mt}(b) + \ln(1/\delta)\big)}\,.
\]
Note that the above confidence interval was already established in \cite[Theorem~1]{cesa2013gang}.
As expected, it depends on $B$, $N$, $b$, and in a decreasing fashion with respect to $\epsilon$.
However, we argue that the above naive choice can be \emph{extremely conservative}.
Indeed, when $b = 0$, MT regression treats tasks independently, see \Cref{prop:extreme_cases}.
Hence, a valid confidence width from~\citep{abbasi2011improved,abbasi2012online,chowdhury2017} is 
$\mathcal{O}\big(B + \sqrt{\gamma_{t}^\text{st}}\big)$, where $\gamma^\text{st}_{t}$ is the single-task  maximum information gain.
Instead, noting that $\gamma^\text{mt}_t(0) = \mathcal{O}\big(N\gamma_t^\text{st}\big)$, see \Cref{prop:bound_info_gain}, the naive choice provides $\beta_t^\text{naive}(0) = \sqrt{N}\cdot \mathcal{O}\big(B + \sqrt{\gamma^\text{st}_{t}}\big)$, which is larger by a factor $\sqrt{N}$.
A similar suboptimality gap of $\sqrt{N}$ can also be proven when $b$ tends to $+\infty$.
Motivated by the above observation, we derive a novel confidence width that is less conservative than $\beta_t^\text{naive}(b)$ for the whole range of possible kernel parameters $b$.\looseness=-1
\smallskip

\begin{figure}[t]
    \centering
    \pgfplotsset{every axis/.append style={
                    label style={font=\normalsize},
                    tick label style={font=\tiny}  
                    }}
\def\N{20}
\def\B{1}
\def\eps{0.4}
\def\t{4}
\begin{tikzpicture}[scale=1,
  declare function={
lambda(\x)=(1+\x*N)/(N*(1+\x));
naive(\x)=\B*sqrt(\N + \x*\N*\eps^2);
bound_first(\x) = sqrt((1+\x*\N)/(1+\x))*(\B + \x*\B*\eps);
bound_second(\x) = \B*sqrt(1/(1+\x)*(1 + \x*\eps)^2 + 2*\x*\N/(1+\x)*(1 + ((1+\x*\eps)^2)/((\x+\N)^2)*\t^2));
improved(\x) = min(min(bound_first(\x), bound_second(\x)), naive(\x));
}
]
 \begin{axis}[
 height=0.3\textwidth,
 width=0.5\textwidth,
   scaled ticks=false,
   xmin=-3,
   xmax=80,
   ymin=0,
   ymax=17,
   xlabel=$b$,
   ylabel= \large $\beta_t$\normalsize$(b)$,
    ytick={\B,\B*sqrt(\N)},
    yticklabels = {\small $1 + \mathcal{O}\big(\sqrt{\gamma_t^\text{st}}\big)$ , \small$\sqrt{N} + \mathcal{O}\big(\sqrt{N\gamma^\text{st}_t}\big)$},
   x label style={at={(axis description cs:0.5,-0.1)},anchor=north},
    y label style={at={(axis description cs:-0.1,.7)},rotate=-90,anchor=south},
   legend pos = south east,
   legend cell align={left},
   name=ax1
 ]

 \addplot[domain=0:100,samples=100, gray, ultra thick] {naive(x)};
 \addplot[domain=0:100, samples=100, red, ultra thick,smooth] {improved(x)};
 \addplot[dashed, domain=0:100,samples=100, black, ultra thick] {bound_first(x)};
\addplot[dotted,domain=0:100,samples=100, black, ultra thick] {bound_second(x)};
   \addlegendentry{naive}
   \addlegendentry{this work}
  \coordinate (c1) at (axis cs:-.5,.5);
  \coordinate (c2) at (axis cs:3,7);
  \draw[ultra thin] (c1) rectangle (axis cs:3,7);
   \node[rotate=5] at (axis cs: 60, 12.1) {\small $\mathcal{O}(\epsilon\sqrt{b})$};
\node[rotate=15] at (axis cs: 55, 15.4) {\small $\mathcal{O}(\epsilon\sqrt{bN})$};
\end{axis}

\begin{axis}[
 height=0.25\textwidth,
 width=0.35\textwidth,
   name=ax2,
   scaled ticks=false,
   xmin=-.05,xmax=3,
   ymin=+.5,ymax=7,
   xlabel=$b$,
   at={($(ax1.south east)+(1.5cm,0.07cm)$)},
   ytick={\B,\B*sqrt(\N)},
    yticklabels = {\scriptsize  , \scriptsize },
    x label style={at={(axis description cs:0.5,-0.08)},anchor=north}
 ]
 
 \addplot[domain=0:3, gray, ultra thick] {naive(x)};
   \addplot[domain=0:3, red, ultra thick] {improved(x)};
 \addplot[dashed, domain=0:3, black, ultra thick] {bound_first(x)};
  \addplot[dotted,domain=0:3, black, ultra thick] {bound_second(x)};
\end{axis}
\draw [ultra thin] (c1) -- (ax2.south west);
\draw [ultra thin] (c2) -- (ax2.north west);
\end{tikzpicture}
     \vspace{-0.2cm}
    \caption{Novel multi-task confidence width $\beta^\text{new}_t(b)$ (see Theorem~\ref{thm:multitask_conf_intervals}) visualized for large and small values of $b$.
    It improves over the naive width $\beta_t^\text{naive}(b)$ by a factor of $\sqrt{N}$ at $b=0$ and as $b\rightarrow +\infty$. Problem parameters were set to $B=1,\epsilon=0.4, N=20,t=4$, and $\gamma_t^\text{mt}(b)=\gamma^\text{st}_t=0$ for all $b$.}
    \label{fig:improved_confidence}
    \vspace{-0.1cm}
\end{figure}
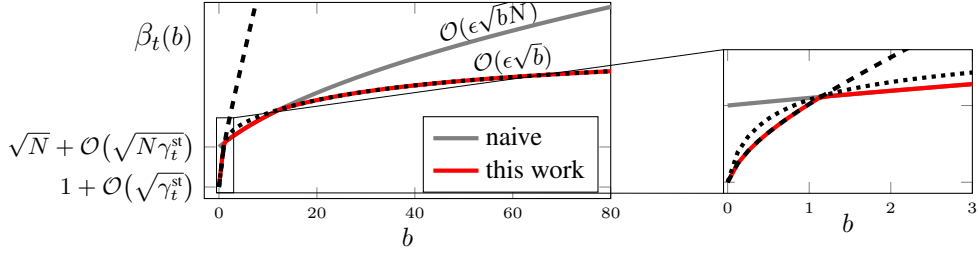

\begin{restatable}[Multitask confidence intervals]{thm}{thminterval}\label{thm:multitask_conf_intervals}
Let $\fmt\colon[N]\times \X \rightarrow \mathbb{R}$ such that for all $i\in [N]$, $f_i \coloneqq \fmt(i,\cdot)$ belongs to the RKHS associated to $k_\X$ and $\|f_i\| \leq B$.
Moreover, let $\mu_t$  and $\sigma_t$ be the regression estimates of \Cref{eq:kernel_regression_mean,eq:kernel_regression_var} with task kernel $k_\T(i,j) =  [K_\textnormal{task}(b)]_{ij}$, parameter $\lambda \in [1/(1+b), 1]$, and noise $\{\xi_\tau\}_{\tau=1}^t$ i.i.d.\ $1$-sub-Gaussian.
Then, with probability at least $1-2\delta$,
\begin{equation*}
    \big|\mu_t(i, x \,|\, b) - \fmt(i,x)\big| \leq \beta_t^\textnormal{new}(b) \cdot \sigma_t(i,x\,|\,b), \qquad \forall\, t \in \mathbb{N}, i \in [N], x \in \mathcal{X}~,
\end{equation*}
\vspace{-1em}
\begin{align*}
\text{where} \quad \beta_t^\textnormal{new}(b) &= \min\Big\{ \beta_t^\textnormal{naive}(b),\, \beta_t^\textnormal{small-$b$}(b),\, \beta_t^\textnormal{large-$b$}(b) \Big\}\,,\\
\beta_t^\textnormal{small-$b$}(b) &= B(1+b\epsilon)\sqrt{\frac{1+bN}{1+b}} + \lambda^{-1/2} \sqrt{2(1+bN)\big(\gamma_t^\textnormal{st} + \ln(N/\delta)\big)}\,,\\
\beta_t^\textnormal{large-$b$}(b) &=  B \sqrt{\frac{(1+b\epsilon)^2}{1+b} +  \frac{2bN}{1+b} + \frac{2b(1+b\epsilon)^2}{N\lambda^2(1+b)^3}\,t^2} + \lambda^{-1/2} \sqrt{2\big(\gamma_t^\textnormal{mt}(b) + \ln(1/\delta)\big)}\,.
\end{align*}
\end{restatable}

The obtained improved confidence width $\beta_t^\textnormal{new}(b)$ is the minimum between three confidence widths, see \Cref{fig:improved_confidence}.
The first one is the naive one $\beta_t^\textnormal{naive}(b)$, obtained by standard arguments as outlined above, while $\beta_t^\textnormal{small-$b$}(b)$ and $\beta_t^\textnormal{large-$b$}(b)$ (dashed and dotted lines in Figure~\ref{fig:improved_confidence}) are novel and useful for small and large values of $b$, respectively.
%
%
Indeed, note that we have $\beta^\text{small-$b$}_t(b) \xrightarrow[]{b \rightarrow 0} \mathcal{O}\big(B+\sqrt{\gamma_t^\text{st}}\big)$, which is the expected single-task confidence width and $\sqrt{N}$ smaller than $\beta_t^\text{naive}(0)$.
Similarly, as $b$ goes to $+\infty$ we have $\beta_t^\text{large-$b$}(b) \overset{b \rightarrow +\infty}{\sim} \mathcal{O}\big(B \sqrt{b \epsilon^2 + 2N + 2\epsilon^2 t^2/N}\big) \overset{b \rightarrow +\infty}{\sim} \mathcal{O}\big(\epsilon B \sqrt{b}\big)$, while $\beta^\text{naive}_t(b) \overset{b \rightarrow +\infty}{\sim} \mathcal{O}\big(\epsilon B \sqrt{Nb} \big)$.
The obtained confidence width is therefore always smaller than the naive one, but also tighter by a factor $\sqrt{N}$ for the extreme choices $b=0$ and $b=+\infty$.

From a technical viewpoint, $\beta_t^\textnormal{small-$b$}$ and $\beta_t^\textnormal{large-$b$}$ are obtained by viewing MT regression as a single-task regression over the inflated features $\tpsi(i, x)$, as also done in~\cite{cesa2013gang}.
However, unlike~\cite{cesa2013gang}, we explicitly leverage the expressions of $\tpsi(i, x\,|\,b)$ and $K_\text{task}(b)$ as functions of $b$, see in particular \Cref{lem:SM_kronecker}, where the structure of $K_\text{task}$ is critical.
%
Moreover, we note that refined widths can be obtained if one has access to task-specific constants $B_i$ and $\epsilon_i$.
For simplicity of exposition, we focus on uniform (over tasks) $B$ and $\epsilon$.  Also, a tighter data-dependent $\beta_t^\textnormal{large-$b$} $ can be utilized as outlined in~\Cref{app:improved_width}. Finally, we remark that the obtained multitask intervals do not require i.i.d.\ data and thus apply to the \emph{adaptive design} setting where data are, e.g., sequentially acquired by the learner, as shown in the next section.\looseness=-1

%
%
%

\section{New Guarantees for Multitask Online Learning}\label{sec:online}

In this section, we show how the improved confidence interval established in \Cref{thm:multitask_conf_intervals} can be used to derive sharp regret guarantees for multitask online learning.
To do so, we also prove novel bounds for the multitask information gain $\gamma_T^\text{mt}(b)$.
For $t=1, 2, \ldots$ the learning protocol is as follows: nature reveals task index $i_t \in [N]$; the learner chooses strategy $x_t \in \X$ and pays $\fmt(i_t, x_t)$; the learner observes the noisy feedback $y_t = \fmt(i_t, x_t) + \xi_t$.
The goal is to minimize for any horizon $T$ the multitask regret
\begin{equation}
\label{eq:regret_online_learning}
\Rmt(T) = \sum_{t=1}^T \max_{x\in \mathcal{X}} \fmt(i_t, x) - \sum_{t=1}^T \fmt(i_t, x_t)\,.
\end{equation}
In the next subsection, we provide a generic algorithm to minimize \eqref{eq:regret_online_learning}.
In particular, we show that naive choices of parameters allow to recover previous approaches with their guarantees, while using the refined confidence width $\beta_t^\text{new}(b)$ derived in \Cref{thm:multitask_conf_intervals} yields significant improvements.

\subsection{Algorithm and regret guarantees}
\label{subsec:mt_algo}

\begin{figure}[t]
\begin{minipage}[t]{0.55\textwidth}
\begin{algorithm}[H]
\caption{\texttt{MT-UCB}}\label{alg:mt_ucb}
\begin{algorithmic}[1]
\REQUIRE Domain $\X$, kernel $k_\X$, number of tasks $N$,\\
\hspace{0.8cm} width functions $\{\beta_t\}_{t \in \mathbb{N}}$\,, parameters $b, \lambda$.
\FOR{t = 1, 2, \ldots}{}
\STATE Observe $i_t$,
\STATE Play $x_t = \argmax_{x \in \X} \, \text{ucb}_{t-1}(i_t, x\,|\,b)$,
\STATE Observe $y_t = \fmt(i_t, x_t) + \epsilon_t$,
\STATE Update $\text{ucb}_t(\cdot, \cdot \,|\,b)$ based on \eqref{eq:kernel_regression_mean},\eqref{eq:kernel_regression_var}, and \eqref{eq:ucb}.
\ENDFOR
\end{algorithmic}
\end{algorithm}
\end{minipage}
\hfill
\renewcommand{\arraystretch}{1.5}
\begin{minipage}[t]{0.41\textwidth}
\begin{table}[H]
\centering
\begin{tabular}{|c|c|c|c|}\hline
{\bf Algorithm} & $\beta_t$ & $b$ & $\lambda$ \\\hline
\texttt{IGP-UCB} \cite{chowdhury2017} & $\beta_t^\text{small-$b$}$ & $0$ & $1$ \\\hline
\texttt{GoB.Lin} \cite{cesa2013gang}  & $\beta_t^\text{naive}$     & $b$ & $1$ \\\hline
This work                             & $\beta_t^\text{new}$       & $b$ & $\frac{N+b}{N+bN}$ \\\hline
\end{tabular}
\vspace{0.3cm}
\caption{Recovering previous works by appropriate choices of $\beta_t$, $b$, and $\lambda$.}
\label{tab:equivalence}
\end{table}
\end{minipage}
\vspace{-0.2cm}
\end{figure}

In line with the online learning literature, our approach is based on the multitask Upper Confidence Bound, defined for any $t \in \mathbb{N}$ as
\begin{equation}\label{eq:ucb}
\text{ucb}_t(i,x\,|\,b) = \mu_t\big(i,x \,|\, b\big) + \beta_t(b) \cdot \sigma_t\big(i,x\,|\,b\big)\,.
\end{equation}
Here $\beta_t \colon \mathbb{R}_+ \rightarrow \mathbb{R}_+$ is a function which assigns a confidence width $\beta_t(b)$ to each kernel parameter $b$.
We consider the general strategy \texttt{MT-UCB} (see \Cref{alg:mt_ucb}) which, at each round $t$ selects $x_t = \argmax_{x \in \X} \text{ucb}_{t-1}(i_t, x\,|\,b)$.
As summarized in \Cref{tab:equivalence}, both the strategy that runs $N$ independent instances of \texttt{IGP-UCB} (one for each task), and \texttt{GoB.Lin} from \cite{cesa2013gang} are particular cases of \texttt{MT-UCB}.
Importantly, whenever $\beta_t(b)$ is set such that $[\mu_t(\cdot,\cdot\,|\,b) \pm \beta_t(b) \cdot \sigma_t(\cdot,\cdot\,|\,b)]$ is a valid confidence interval for $\fmt(\cdot,\cdot)$, the regret of \texttt{MT-UCB} can be controlled through the following lemma.\looseness-1
\smallskip

\begin{restatable}{lem}{lemgeneric}\label{lem:generic_bound}
Suppose that $\lambda \ge (N+b)/(N+bN)$, and that for all tasks $i$, point $x$, and time $t$, we have $\fmt(i, x\,|\,b) \in [\,\mu_t(i,x\,|\,b) \pm \beta_t(b) \cdot \sigma_t(i,x\,|\,b)\,]$.
Then, the multitask regret of \textnormal{\texttt{MT-UCB}} satisfies
\[
\Rmt(T) \le 4\,\beta_T(b) \sqrt{\lambda\,T\gamma_T^\textnormal{mt}(b)}\,.
\]
\end{restatable}

The main novelty of \Cref{lem:generic_bound} is that the right-hand side scales with $\lambda^{1/2}$, which might be chosen smaller than $1$.
This improvement is due to the fact that multitask posterior variances are smaller than $(N+b)/(N+bN) \le 1$.
The right-hand side also depends on the multitask information gain $\gamma_T^\text{mt}(b)$, which is nontrivial to compute or upper bound.
In the next proposition, we provide practical upper bounds of $\gamma_T^\text{mt}(b)$, in terms of the kernel parameter $b$ and the single-task information gain $\gamma_T^\text{st}$.
\smallskip

\begin{restatable}{prop}{propinfogain}\label{prop:bound_info_gain}
Let $\lambda \le 1$, $N \ge 2$, and $T_i \ge 1$ for all $i \in [N]$.
Then, for any $b \ge 0$, we have
\[
\gamma_T^\textnormal{mt}(b) \le N \gamma_T^\textnormal{st} + \frac{b}{2}\left(T - \frac{N}{4}\right) - \frac{T}{2}\ln(1+b) \hspace{2.5em}  \text{and} \hspace{2.5em} \gamma^\textnormal{mt}_T(b) \leq \gamma_T^\textnormal{st} + \frac{T}{\lambda b}\,.
\]
\end{restatable}

We can now combine \Cref{thm:multitask_conf_intervals}, \Cref{lem:generic_bound}, and \Cref{prop:bound_info_gain} to obtain our main result: a bound on the multitask regret of \texttt{MT-UCB} run with the confidence width $\beta_t^\text{new}$ from \Cref{thm:multitask_conf_intervals} and a specific $\lambda$.
\smallskip

\begin{restatable}{thm}{thmmain}\label{thm:main_thm_regrets}
Assume that $B \ge 1$, and that \textnormal{\texttt{MT-UCB}} is run with $\beta_t = \beta^\textnormal{new}_t$ from \Cref{thm:multitask_conf_intervals}, and $\lambda = (N+b)/(N+bN)$.
Let $b = N/\epsilon^2$ if $T \le N$, $b=1/\epsilon^2$ if $T \ge N$ and $\epsilon \le N^{-1/4}T^{-1/2}$, and $b=0$ otherwise.
Let $R^\textnormal{st}(T) = B\sqrt{T\gamma^\textnormal{st}_T} + \sqrt{T\gamma_T^\textnormal{st}\big(\gamma_T^\textnormal{st} + \ln(1/\delta)\big)}$ be the single task regret bound achieved by \textnormal{\texttt{IGP-UCB}} (up to constant factors).
Then, there exists a universal constant $C$ such that with probability $1-2\delta$ we have (up to $\log N$ factors)
\begin{align*}
\textstyle\Rmt(T) \le C \min\Big\{\sqrt{N}R^\textnormal{st}(T) ~\,,~\, & R^\textnormal{st}(T) + \epsilon B T^{3/2}\Big(\sqrt{\gamma^\textnormal{st}_T + \ln(1/\delta)} + \epsilon\sqrt{T}\Big)~\,,\\
&R^\textnormal{st}(T) + \epsilon BT\sqrt{N}\Big(\sqrt{\gamma^\textnormal{st}_T + \ln(1/\delta)} + \epsilon \sqrt{NT}\Big)\,\Big\}\,.
\end{align*}
\end{restatable}

The regret bound of \Cref{thm:main_thm_regrets} is the minimum between three bounds, obtained exploiting the three different regimes of the confidence width $\beta^\text{new}_t$ derived in \Cref{thm:multitask_conf_intervals} (see \Cref{fig:improved_confidence}).
The \textbf{first bound} is obtained using $\beta_t^\text{new} \le \beta_t^\text{small-$b$}$, and shows that our approach cannot be worse than independent learning.
Indeed, it can be checked that, when facing $N$ tasks, the regret of running $N$ independent instances of \texttt{IGP-UCB} can be bounded by $\sqrt{N}$ times the single-task regret bound of \texttt{IGP-UCB}, that we denoted by $R^\text{st}(T)$.
Note however that our analysis slightly differs, insofar as we leverage the multitask information gain, while the independent analysis uses Jensen's inequality to aggregate the individual bounds, see \Cref{app:online_learning} for details.
Note finally that we are able to recover this bound as $\beta_t^\text{small-$b$}$ is tight at $b=0$, unlike $\beta_t^\text{naive}$.
The \textbf{second bound} uses $\beta_t^\text{new} \le \beta_t^\text{large-$b$}$ and consists of two terms: the single task regret bound and an additional term that scales with the task deviation $\epsilon$.
When the latter is small, i.e., when tasks are similar, the dominant term is $R^\text{st}(T)$, as if only one task were solved.
The \textbf{third bound} is similar, but obtained using $\beta_t^\text{new} \le \beta_t^\text{naive}$ and is useful when $T \geq N$.
In contrast with the independent bound, which does not exploit the task structure, the last two bounds show that multitask learning is always beneficial when the horizon $T$ (and thus the additional $\epsilon$-related term) is small.
As expected, this is particularly true when the number of tasks $N$ is large: while the independent bound increases, the second bound \emph{does not depend on $N$}.
On the other hand, one can note that the condition on $\epsilon$ to improve over independent becomes more constraining as the horizon $T$ increases.
This suggests that the benefit of multitask may vanish with the number of available points per task, an observation which is well-known by practitioners.
As far as we know, this work is the first one to provide theoretical evidence of such phenomenon.

We conclude this section by comparing \Cref{thm:main_thm_regrets} to existing results.
As already mentioned in the above discussion, independent \texttt{IGP-UCB} is a particular case of \texttt{MT-UCB}, such that we cannot be worse than the independent approach.
We incidentally recover its regret bound as the first bound in the minimum of \Cref{thm:main_thm_regrets}.
Regarding \texttt{GoB.Lin}, since it is also a specific instance of \texttt{MT-UCB} (for $\beta_t = \beta_t^\text{naive}$ and $\lambda=1$), \Cref{lem:generic_bound} allows to recover its regret bound \cite[Theorem~1]{cesa2013gang}.
\smallskip

\begin{cor}[Regret of \texttt{GoB.Lin}~\cite{cesa2013gang}]
For any $b$, the multitask regret of \textnormal{\texttt{GoB.Lin}} using parameter $b$ satisfies with probability $1-\delta$
\begin{equation}\label{eq:bound_gob}
\Rmt(T) \le 4 \beta^\textnormal{naive}_T(b) \sqrt{T\gamma_T^\textnormal{mt}(b)} \le 6\left(B\sqrt{N(1+b\epsilon^2)} + \sqrt{\gamma^\textnormal{mt}_T(b) + \ln(1/\delta)}\right) \sqrt{T\gamma^\textnormal{mt}_T(b)\,}\,.
\end{equation}
\end{cor}
If tasks are similar, i.e., when $\epsilon \ll 1$, bound \eqref{eq:bound_gob} suggests to choose $b > 0$; this does not impact too much the first term, but makes $\gamma^\textnormal{mt}_T(b)$ smaller.
However, we recall that the above bound instantiated with $b=0$ does not recover the independent bound.
It is instead $\sqrt{N}$ bigger, since $\beta_t^\text{naive}$ is not tight at $b=0$.
Hence, the \texttt{Gob.Lin} analysis is not sufficient to show that multitask learning improves over independent learning.
Our refined analysis, which uses instead $\beta_t^\text{new}$, closes this gap.

\subsection{Adapting to unknown task similarity}\label{sec:ada_mtucb}

In this section, we consider the case where parameter $\epsilon$ (i.e., a bound on the task deviation from the average, see~\eqref{eq:def_eps}) is a-priori unknown. Despite this challenge, we show that the regret bound of Theorem~\ref{thm:main_thm_regrets} can be approximately attained using an adaptive procedure, \texttt{AdaMT-UCB} (Algorithm~\ref{alg:adamt-ucb}), relegated to \Cref{app:adamtucb} due to space limitations. The proposed approach is inspired by the model selection scheme of~\cite[Section~7]{Pacchiano2020RegretBB} with a few important modifications that we will outline at the end of this section. \texttt{AdaMT-UCB} considers a plausible set of parameters $\mathcal{E} = \{e_1, \ldots, e_{|\mathcal{E}|}\} \subset (0,2]$ and, for each $e \in \mathcal{E}$, initializes an instance of the \texttt{MT-UCB} algorithm with parameters set according to Theorem~\ref{thm:main_thm_regrets} assuming $\epsilon=e$.
We denote such an instance as $\texttt{MT-UCB}(e)$. Moreover, we use the notation $\text{ucb}_t^e$ to denote the upper confidence bounds constructed by $\texttt{MT-UCB}(e)$.
We assume the existence of some $e\in \mathcal{E}$ such that $e \geq \epsilon$, so that at least one of the learners is \emph{well-specified} (i.e., its confidence bounds contain $\fmt$ with high probability). Our goal is to incur a regret which grows as the regret of the learner with the smallest $e$ such that $e \geq \epsilon$, since the smaller the $e$ the smallest the regret bound (see~\Cref{thm:main_thm_regrets}), as long as $e$ is a valid upper bound for $\epsilon$. Let us identify  with $e^\star$ such learner.\looseness-1

At each round $t$, \texttt{AdaMT-UCB} uses learner $e_t = \min \mathcal{E}$, and play the action $x_t$ suggested by it, i.e., the maximizer of $\text{ucb}_t^{e_t}(i_t, \cdot)$. Then, all $\texttt{MT-UCB}(e)$ learners are updated based on the observed reward. In the meantime, a \emph{misspecification test}
is carried out to check whether learner $e_t$ is well-specified. It compares the cumulative reward and the believed regret of learner $e_t$, with a lower confidence estimate on such reward according to the other learners. 
If the test triggers, learner $e_t$ is misspecified with high probability and gets removed from $\mathcal{E}$.
A \emph{new epoch} starts with the new set $\mathcal{E}$.
Let $\overline{\Rmt_\star}(T)$ denote the regret bound (Theorem~\ref{thm:main_thm_regrets}) of learner $e^\star$ had it been chosen from round $0$.
We can state the following.
\looseness=-1
\smallskip

\begin{restatable}{thm}{thmadamtucb}\label{cor:adamt-ucb} Assume that there exists $e\in \mathcal{E}$ such that $e \geq \epsilon$, and let $M$ be the number of learners $e\in \mathcal{E}$ such that $e < \epsilon$ (i.e., the number of misspecified learners in $\mathcal{E}$). The regret of \textnormal{\texttt{AdaMT-UCB}} satisfies with high probability $\Rmt(T) = \mathcal{O}\big(\sqrt{M+1} \cdot \overline{\Rmt_\star}(T)\big).$
\end{restatable}
Clearly, the number $M$ of misspecified learners is not known in advance but is always less than $|\mathcal{E}|$.
Note that when $\epsilon=0$, we have $M=0$ and we recover the single task regret bound.
Moreover, given $\rho \le 1$, we show in~\Cref{app:adamtucb} that one can attain a multiplicative accuracy $\rho$ over $\epsilon$, assuming that $\epsilon \geq \epsilon_\text{min}>0$, through an exponential grid with $M$ being polylogarithmic in $1/\rho$ and $1/\epsilon_\text{min}$.

\textbf{Relation with the approach of~\citep{Pacchiano2020RegretBB}.} Compared to~\citep[Section~7]{Pacchiano2020RegretBB}---where the goal is to adapt to an unknown features' dimension---the set of learners considered in \texttt{AdaMT-UCB} share \emph{the same dimension $d$}. This allows us to exploit the following two novelties with respect to~\citep{{Pacchiano2020RegretBB}}: (1) \emph{all} learners are updated from the data gathered from learner $i_t$ (Line~6 in Algorithm~\ref{alg:adamt-ucb}), and (2) the lower confidence bounds $L^e$ in the misspecification test (Line~8) are all computed using action $x_t$ (i.e., the action recommended by learner $i_t$), as opposed to using the actions recommended by each learner $e$. Both these points are only applicable to our setting, leading to a simpler regret analysis.

\section{Multitask Active Learning}
\label{sec:active_learning}

The goal of the online learning setup of Section~\ref{sec:online} is to optimize the tasks sequentially revealed by nature. In some situations (e.g., in~\cite{lugosi2009} or the drug discovery problem considered in Section~\ref{sec:experiments}), however, we care about the performance of multiple tasks \emph{simultaneously}, to eventually learn the best strategy for each one of them. Moreover, we ought to do so with minimal interactions $T$, i.e., minimizing the queries of the function $\fmt$. We capture this by the following \emph{active learning} protocol. 

\textbf{Learning protocol and regret.}
At each round $t$, the learner: chooses a strategy $\{ x_t^i, i \in [N]\}$ \emph{for each task}, chooses \emph{which task} $i_t \in [N]$ to query, and observes the noisy feedback $y_t = \fmt(i_t, x_t) + \xi_t$. The learner's goal is to minimize the \emph{active learning} regret:
\begin{equation*}\label{eq:regret_active_learning}
    \Rmt_\textnormal{AL}(T) = \sum_{t=1}^T \frac{1}{N}\sum_{i=1}^N \max_{x\in \mathcal{X}} \fmt(i,x) - \sum_{t=1}^T \frac{1}{N} \sum_{i=1}^N \fmt(i,x_t^i)\,.
\end{equation*}
Compared to the online learning regret of ~\Cref{eq:regret_online_learning}, the learner's performance at each round is here measured by the average reward coming from \emph{each} task (as opposed to just the task presented by nature). Moreover, compared to online learning, the learner faces the additional challenge of choosing---at each round---from which task information should be gathered. Intuitively, more difficult (or informative) tasks should be queried more often to ensure $\Rmt_\textnormal{AL}(T)$ grows sublinearly. To the best of our knowledge, the above protocol and regret notion are novel in the multitask literature.

\begin{wrapfigure}{R}{0.48\textwidth}
\vspace{-2.em}
\begin{minipage}[H]{0.48\textwidth}
\begin{algorithm}[H]
\caption{\texttt{MT-AL}}
\begin{algorithmic}
\FOR{t=1,\ldots, T}{}
   \STATE $x_t^i  = \arg\max_{x \in \mathcal{X}} \text{ucb}_{t-1}(i,x), \, \forall i \in  [N] $
    \STATE $i_t  = \arg\max_{i \in [N]} \beta_{t-1}^i\sigma_{t-1}(i, x_t^i)  $
    \STATE Observe: $y_t = \fmt(i_t, x_t^{i_t}) + \xi_t$
 \STATE Update $\text{ucb}_t(\cdot,\cdot)$ and $\sigma_t(\cdot,\cdot)$ based on \\ observations.
\ENDFOR
\end{algorithmic}\label{alg:mt_al}
\end{algorithm}
\end{minipage}
\vspace{-1em}
\end{wrapfigure}

In Algorithm~\ref{alg:mt_al} we present \texttt{MT-AL}, an efficient strategy that ensures sublinear active learning regret. 
Like in \texttt{MT-UCB}, \texttt{MT-AL} constructs confidence intervals around $\fmt$ and, at each round, select strategy 
$x_t^i  = \arg\max_{x \in \mathcal{X}} \text{ucb}_{t-1}(i,x)$ for each task $i\in [N]$. When it comes to selecting which task to query, \texttt{MT-AL} selects $i_t \in \arg\max_{i \in [N]} \beta_{t-1}^i\sigma_{t-1}(i, x_t^i) $, i.e., the task for which the believed optimizer $x_t^i$ is subject to maximal uncertainty (we use generic task-dependent widths $\beta_t^i$ for completeness). This rule, also known as \emph{uncertainty sampling} in the literature~\cite{settles2009active}, intuitively makes sure the learner can control the regrets for the tasks not queried and leads to the following theorem.

\begin{restatable}{thm}{thmactivelearning}\label{thm:active_singlecluster}
Suppose that for all tasks $i$, point $x$, and time $t$, we have that $\fmt(i, x) \in [\,\mu_t(i,x) \pm \beta_t^i \cdot \sigma_t(i,x)\,]$. Then, 
the \textnormal{\texttt{MT-AL}} algorithm ensures the active learning regret is bounded by\looseness=-1 
    \begin{equation*}
        \Rmt_\textnormal{AL}\leq 2\sum_{t=1}^T \beta_t^{i_t} \sigma_t(i_t, x_t^{i_t})\,,
    \end{equation*}
    where $\{i_t\}$ is the sequence of queried tasks and $\{x_t^{i_t}\}$ the strategies selected for each of them.
\end{restatable}
The above bound only relies on \textnormal{\texttt{MT-AL}} utilizing valid intervals around $\fmt$ and thus applies more broadly than our agnostic MT regression, e.g., when such intervals are constructed using a known multitask kernel $k\big((i,x),(i',x')\big)$). However, \Cref{thm:active_singlecluster} shows the active learning regret heavily depends on the constructed intervals, similar to online learning. In \textnormal{\texttt{MT-AL}}, these are additionally utilized for deciding which task to query at each round. When specialized to our agnostic MT kernel and improved confidence, we obtain the following.\looseness=-1

\begin{restatable}{cor}{corolactivelearning}\label{cor:active_learning}
Let \textnormal{\texttt{MT-AL}} utilize the MT regression estimates of Eq.~\eqref{eq:kernel_regression_mean}-\eqref{eq:kernel_regression_var} with parameters set according to Theorem~\ref{thm:main_thm_regrets}. Moreover, let $\overline{\Rmt}(T)$ be the bound on the online learning regret obtained in Theorem~\ref{thm:main_thm_regrets}. Then, with high probability, we have $\Rmt_\textnormal{AL}(T) \leq \overline{\Rmt}(T)$.
\end{restatable}

Thus, \texttt{MT-AL} ensures the active learning regret is always bounded by its online learning counterpart. Moreover, the same considerations as in~\Cref{thm:main_thm_regrets} apply also here, regarding the benefit of multitask learning over independent single-task regression for instance. 
\section{Experiments}
\label{sec:experiments}

The goal of our experiments is to evaluate the effectiveness of the studied MT regression, and in particular of the improved confidence intervals obtained in Section~\ref{sec:mt_setting}, both in online learning and active learning setups. We utilize the following synthetic and real-world data. 

\emph{Synthetic data:}  
We generate tasks of the form $f_i = (1-\delta)\cdot \bar{f} + \delta\cdot f_\text{dev}^i, i\in [N]$, where $\bar{f}, f_\text{dev}^i$ are random unit vectors representing a common model and individual deviations, respectively. Moreover, actions consist of $10^4$ vectors $x\in \mathbb{R}^d$ from the sphere of radius 10. Observation noise is unit normal.

\emph{Drug discovery MHC-I data~\cite{widmer2010inferring}:}  The goal is to discover the peptides with maximal binding affinity to each Major
Histocompatibility Complex class-I (MHC-I) allele. The dataset from~\cite{widmer2010inferring} contains the standardized binding affinities (IC$_{50}$ values) of different
peptide candidates to the MHC-I alleles (tasks). For each allele, the dataset contains $\sim1000$ peptides represented as $x\in \mathbb{R}^{45}$ feature vectors. For our experiments, we utilize the $5$ alleles A-$\{0201, 0202, 0203, 2301, 2402 \}$, since they were shown in~\cite{widmer2010inferring} to share binding similarity. Note that such a problem falls into our multitask active learning setup, since we would like to retrieve the best peptide for each allele minimizing the number of interactions (i.e., lab experiments). Nevertheless, we also consider its online learning analog where we care about finding the best peptides for each revealed allele.

\begin{figure}[t]
        \centering
        \hspace{-1em}
        \begin{subfigure}[b]{0.5\textwidth}
        \includegraphics[width=.49\textwidth]{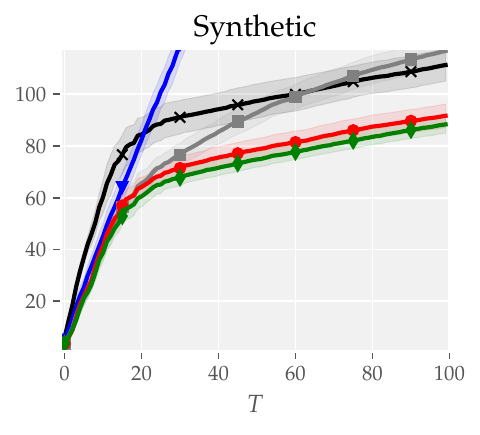}
\includegraphics[width=.49\textwidth]{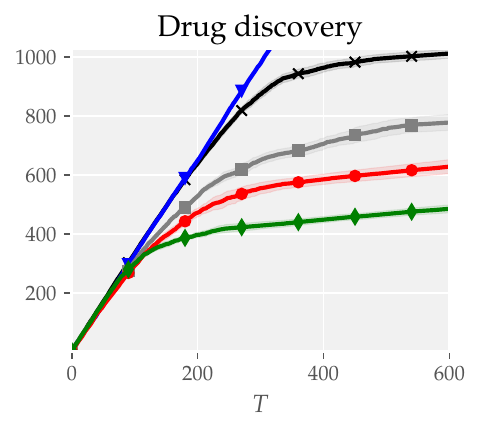}
\includegraphics[width=1\textwidth]{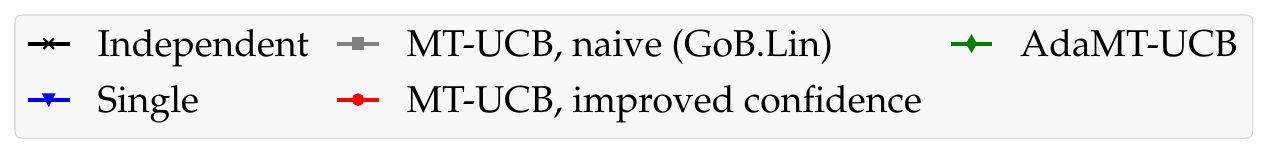}
        \caption{Online learning}
        \end{subfigure}
        \begin{subfigure}[b]{0.5\textwidth}
        \includegraphics[width=.49\textwidth]{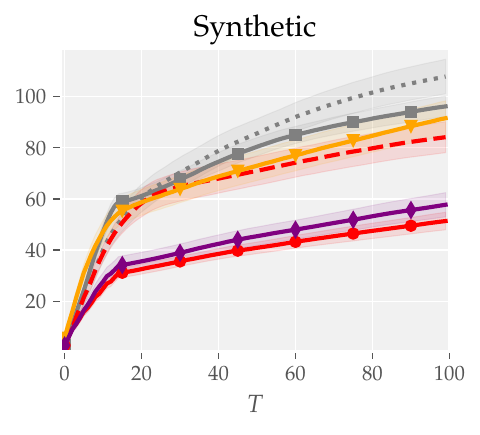}
\includegraphics[width=.49\textwidth]{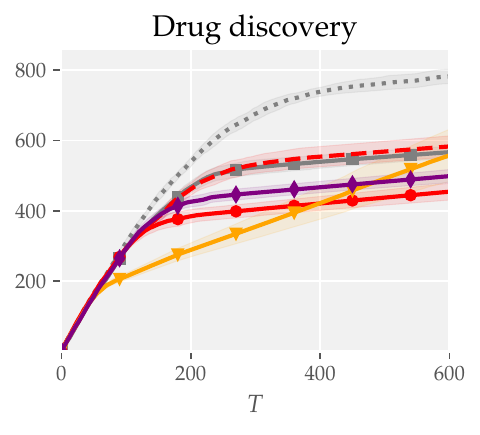}
\includegraphics[width=1\textwidth]{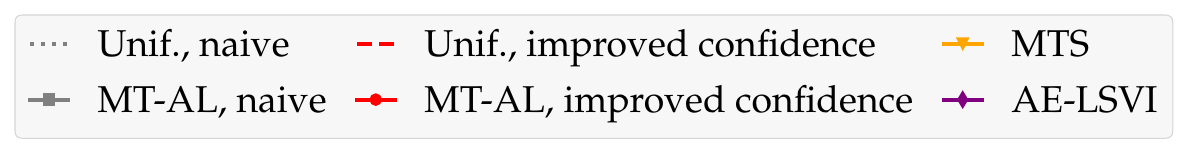}
        \caption{Active learning}
        \end{subfigure}
        \caption{Online and active learning regrets on synthetic and drug discovery MHC-I data, respectively. When utilizing the improved confidence intervals,
        \texttt{MT-UCB} and \texttt{MT-AL} outperform the other baselines.\looseness=-1}
          \vspace{-1.em}
\label{fig:fig_online_and_active}
    \end{figure}

\textbf{Online learning.} At each round $t$, a random task $i_t \in [N]$ is observed and point $x_t$ is selected according to the following baselines: (1) \emph{Independent}, which runs $N$ independent \texttt{IGP-UCB}~\cite{chowdhury2017} algorithms (corresponding to \texttt{MT-UCB} with $b=0$), (2) \emph{Single}, which treats all tasks to be the same and runs a unique single-task \texttt{IGP-UCB} (corresponding to \texttt{MT-UCB} with $b=+\infty$), (3) \texttt{MT-UCB} which utilizes an appropriate parameter $0<b<\infty$ as well as a bound on the tasks similarity $\epsilon$ (for synthetic data this can be exactly computed, while for MHC-I data we use $\epsilon=0.3$) and utilizes the \emph{naive} (i.e., \texttt{Gob.Lin}) or \emph{improved} confidence bounds, and (4) \texttt{AdaMT-UCB} which is run with the same $b$ but uses the set of plausible deviations $\mathcal{E} = \{.1, .2, \ldots, 1\}$ instead of knowing the true $\epsilon$. For choosing $b$, we sweep over possible values and select the best-performing one, keeping it fixed for all the baselines. 

\textbf{Active learning.} We follow the multitask active learning setup of Section~\ref{sec:active_learning}. All baselines utilize confidence intervals from the agnostic MT regression of Section~\ref{sec:mt_setting}, where $\epsilon$ and $b$ are chosen as for online learning. Moreover, they all utilize the improved confidence intervals, unless otherwise specified. We compare: (1) \emph{Unif.} which chooses the task $i_t$ to be queried uniformly at random (but still selects $x_t^i \in \arg\max_x \ucb_t^i(i, x)$) and employs the \emph{naive} or the \emph{improved} confidence intervals, the offline contextual Bayesian optimization baselines (2) \texttt{MTS}~\cite{char2019offline} and (3) \texttt{AE-LSVI}~\cite{li2022near}, and (4) \texttt{MT-AL} which utilizes the \emph{naive} or the \emph{improved} confidence intervals.  

We report the cumulative regret (online and active learning, respectively) of the considered baselines in Figure~\ref{fig:fig_online_and_active}, averaged over 5 runs. For the synthetic data, we report results for $d=4, N=5,\delta=0.4$, but provide a full set of experiments for different parameters in Appendix~\ref{app:additional_exps}. 
In Figure~\ref{fig:fig_online_and_active}~(a), both \texttt{MT-UCB} and \texttt{AdaMT-UCB} lead to superior performance compared to the \emph{Independent} and \emph{Single} baselines, demonstrating the benefits of MT regression. In addition, 
the improved confidence intervals significantly outperform the naive ones. Moreover, we observe \texttt{AdaMT-UCB} achieves comparable (sometimes even better, see Appendix~\ref{app:additional_exps}) performance to \texttt{MT-UCB}.
Indeed, instead of using a conservative choice of $\epsilon$, the misspecification test (Line~8 of Algorithm~\ref{alg:adamt-ucb}) of \texttt{AdaMT-UCB} allows to use a smaller $\epsilon$ and only increase it when there is evidence that the constructed intervals do not contain the true tasks. In active learning (Figure~\ref{fig:fig_online_and_active}~(b)), we observe \texttt{MT-AL} has a significant advantage over the uniform sampling baselines and \texttt{MTS}, while performing comparably to \texttt{AE-LSVI} (both methods are similar as discussed in~\Cref{app:compare_AELSVI}). Moreover, its regret is bounded by the online learning regret of \texttt{MT-UCB}, conforming with~\Cref{thm:active_singlecluster}. Importantly, the improved confidence intervals play a crucial role also here and enable a drastic performance improvement compared to the naive ones.\looseness=-1

\vspace{-.5em}
\section{Future Directions}
\vspace{-.5em}

We believe this paper opens up several future research directions. 
The derived confidence intervals, as well as our analysis of the multitask information gain, heavily exploit the structure of the task Gram matrix $K_\text{task}(b)$, see \Cref{eq:A_clique}. However, it remains unclear whether these can be extended to more general kernels. According to the graph perspective of~\cite{evgeniou2005learning}, $K_\text{task}(b)$ can be seen as $K_\text{task}(b) = I_N + L(b)$, where $L(b) \in \mathbb{R}^{N \times N}$ is the Laplacian matrix of a \emph{clique} graph with vertices $[N]$ and edge weight $b$. 
Hence, it would be interesting to extend our results to different graph structures.
Furthermore, we believe the proposed multitask confidence intervals hold potential for various related domains., e.g., to assess uncertainty in safety-critical systems~\cite{camilleri2022active}, or to balance exploration-exploitation in multitask reinforcement learning~\cite{vithayathil2020survey}. In such applications, the introduced notion of active learning regret can serve as measure for the overall sample-efficiency.

\section*{Acknowledgements}

This work was partially supported by ELSA (European Lighthouse on Secure and Safe AI) funded by the European Union under grant agreement No. 101070617, 
the ELISE (European Learning and Intelligent Systems Excellence) project EU Horizon 2020 ICT-48 research and innovation action under grant agreement No. 951847, and the FAIR (Future Artificial Intelligence Research) project, funded by the NextGenerationEU program within the PNRR-PE-AI scheme.

\bibliographystyle{apalike}
\bibliography{ref}

\begin{thebibliography}{}

\bibitem[Abbasi-Yadkori, 2012]{abbasi2012online}
Abbasi-Yadkori, Y. (2012).
\newblock Online learning for linearly parametrized control problems.
\newblock {\em PhD thesis, Edmonton, Alberta, Canada}.

\bibitem[Abbasi-Yadkori et~al., 2011]{abbasi2011improved}
Abbasi-Yadkori, Y., P{\'a}l, D., and Szepesv{\'a}ri, C. (2011).
\newblock Improved algorithms for linear stochastic bandits.
\newblock {\em Advances in Neural Information Processing Systems (NeurIPS)},
  24.

\bibitem[Camilleri et~al., 2022]{camilleri2022active}
Camilleri, R., Wagenmaker, A., Morgenstern, J.~H., Jain, L., and Jamieson,
  K.~G. (2022).
\newblock Active learning with safety constraints.
\newblock {\em Advances in Neural Information Processing Systems (NeurIPS)}.

\bibitem[Caruana, 1997]{caruana1997multitask}
Caruana, R. (1997).
\newblock Multitask learning.
\newblock {\em Machine learning}, 28:41--75.

\bibitem[Cavallanti et~al., 2010]{cavallanti2010linear}
Cavallanti, G., Cesa-Bianchi, N., and Gentile, C. (2010).
\newblock Linear algorithms for online multitask classification.
\newblock {\em Journal of Machine Learning Research}, 11:2901--2934.

\bibitem[Cesa-Bianchi et~al., 2013]{cesa2013gang}
Cesa-Bianchi, N., Gentile, C., and Zappella, G. (2013).
\newblock A gang of bandits.
\newblock {\em Advances in Neural Information Processing Systems (NeurIPS)},
  26.

\bibitem[Char et~al., 2019]{char2019offline}
Char, I., Chung, Y., Neiswanger, W., Kandasamy, K., Nelson, A.~O., Boyer, M.,
  Kolemen, E., and Schneider, J. (2019).
\newblock Offline contextual {bayesian} optimization.
\newblock {\em Advances in Neural Information Processing Systems (NeurIPS)},
  32.

\bibitem[Chowdhury and Gopalan, 2017]{chowdhury2017}
Chowdhury, S.~R. and Gopalan, A. (2017).
\newblock On kernelized multi-armed bandits.
\newblock In {\em International Conference on Machine Learning (ICML)}, page
  844–853.

\bibitem[Chowdhury and Gopalan, 2021]{chowdhury2021no}
Chowdhury, S.~R. and Gopalan, A. (2021).
\newblock No-regret algorithms for multi-task {bayesian} optimization.
\newblock In {\em International Conference on Artificial Intelligence and
  Statistics (AISTATS)}, pages 1873--1881.

\bibitem[Ciliberto et~al., 2015]{ciliberto2015learning}
Ciliberto, C., Rosasco, L., and Villa, S. (2015).
\newblock Learning multiple visual tasks while discovering their structure.
\newblock In {\em Proceedings of the IEEE/CVF Conference on Computer Vision and
  Pattern Recognition}, pages 131--139.

\bibitem[Collobert and Weston, 2008]{collobert2008unified}
Collobert, R. and Weston, J. (2008).
\newblock A unified architecture for natural language processing: Deep neural
  networks with multitask learning.
\newblock In {\em International Conference on Machine learning (ICML)}, pages
  160--167.

\bibitem[Dekel et~al., 2007]{dekel2007online}
Dekel, O., Long, P.~M., and Singer, Y. (2007).
\newblock Online learning of multiple tasks with a shared loss.
\newblock {\em Journal of Machine Learning Research}, 8(10).

\bibitem[Evgeniou et~al., 2005]{evgeniou2005learning}
Evgeniou, T., Micchelli, C.~A., Pontil, M., and Shawe-Taylor, J. (2005).
\newblock Learning multiple tasks with kernel methods.
\newblock {\em Journal of Machine Learning Research}, 6(4).

\bibitem[Kairouz et~al., 2021]{kairouz2021advances}
Kairouz, P., McMahan, H.~B., Avent, B., Bellet, A., Bennis, M., Bhagoji, A.~N.,
  Bonawitz, K., Charles, Z., Cormode, G., Cummings, R., et~al. (2021).
\newblock Advances and open problems in federated learning.
\newblock {\em Foundations and Trends{\textregistered} in Machine Learning},
  14(1--2):1--210.

\bibitem[Krause and Ong, 2011]{krause2011contextual}
Krause, A. and Ong, C. (2011).
\newblock Contextual {Gaussian} process bandit optimization.
\newblock {\em Advances in Neural Information Processing Systems (NeurIPS)},
  24.

\bibitem[Li et~al., 2023]{li2022near}
Li, X., Mehta, V., Kirschner, J., Char, I., Neiswanger, W., Schneider, J.,
  Krause, A., and Bogunovic, I. (2023).
\newblock Near-optimal policy identification in active reinforcement learning.
\newblock {\em International Conference on Learning Representations (ICRL)}.

\bibitem[Liu et~al., 2019]{liu2019end}
Liu, S., Johns, E., and Davison, A.~J. (2019).
\newblock End-to-end multi-task learning with attention.
\newblock In {\em Proceedings of the IEEE/CVF conference on computer vision and
  pattern recognition}, pages 1871--1880.

\bibitem[Lugosi et~al., 2009]{lugosi2009}
Lugosi, G., Papaspiliopoulos, O., and Stoltz, G. (2009).
\newblock Online multi-task learning with hard constraints.
\newblock In {\em Conference on Learning Theory (COLT)}.

\bibitem[Pacchiano et~al., 2020]{Pacchiano2020RegretBB}
Pacchiano, A., Dann, C., Gentile, C., and Bartlett, P.~L. (2020).
\newblock Regret bound balancing and elimination for model selection in bandits
  and rl.
\newblock {\em ArXiv}, abs/2012.13045.

\bibitem[Settles, 2009]{settles2009active}
Settles, B. (2009).
\newblock Active learning literature survey.
\newblock Technical report, University of Wisconsin-Madison Department of
  Computer Sciences.

\bibitem[Sheldon, 2008]{sheldon2008graphical}
Sheldon, D. (2008).
\newblock Graphical multi-task learning.
\newblock Technical report, Cornell University, Ithaca, NY.

\bibitem[Srinivas et~al., 2010]{srinivas2010gaussian}
Srinivas, N., Krause, A., Kakade, S., and Seeger, M. (2010).
\newblock Gaussian process optimization in the bandit setting: No regret and
  experimental design.
\newblock In {\em International Conference on Machine Learning (ICML)}.

\bibitem[Vithayathil~Varghese and Mahmoud, 2020]{vithayathil2020survey}
Vithayathil~Varghese, N. and Mahmoud, Q.~H. (2020).
\newblock A survey of multi-task deep reinforcement learning.
\newblock {\em Electronics}, 9(9):1363.

\bibitem[Widmer et~al., 2010]{widmer2010inferring}
Widmer, C., Toussaint, N.~C., Altun, Y., and R{\"a}tsch, G. (2010).
\newblock Inferring latent task structure for multitask learning by multiple
  kernel learning.
\newblock {\em BMC bioinformatics}, 11(8):1--8.

\end{thebibliography}

\appendix
\onecolumn

{\centering  \Large\textbf{Supplementary Material} \\
}
\vspace{1em}
We gather here the technical proofs and additional results complementing the main paper.

\section{Multitask Regression and Improved Confidence Intervals}\label{app:mt_regression}

\paragraph{Notation.}
In order to coincide with the notation of \cite{cesa2013gang}, we introduce the matrix $A(b) \in \mathbb{R}^{N \times N}$ such that $A^{-1}(b) \coloneqq K_\text{task}(b)$. For all $b \ge 0$ we have
\begin{alignat*}{2}
A(b) &= (1+b)I_N - b \frac{\mathbbm{1}_N\mathbbm{1}_N^\top}{N}\,,&  A(b)^{1/2} &= \sqrt{1+b}\,I_N + \big(1 - \sqrt{1+b}\big) \frac{\mathbbm{1}_N \mathbbm{1}_N^\top}{N}\,,\\
A(b)^{-1} &= \frac{1}{1+b} I_N + \frac{b}{1+b} \frac{\mathbbm{1}_N\mathbbm{1}_N^\top}{N}\,,& \hspace{1cm} A(b)^{-1/2} &= \frac{1}{\sqrt{1+b}} I_N + \left(1 - \frac{1}{\sqrt{1+b}}\right)\frac{\mathbbm{1}_N \mathbbm{1}_N^\top}{N}\,.
\end{alignat*}
Note that in the following we often drop the dependence in $b$ and use only $A$, $A^{-1}$, $A^{1/2}$, $A^{-1/2}$.


\subsection{Further intuitions about the MT regression of  Section \ref{sec:mt_regression}}\label{app:intuition_regression}

Given the history of measurements $\{(i_s, x_s), y_s\}_{s=1}^t$, the agnostic MT regression considered in \Cref{sec:mt_regression} estimates $\fmt$ (or equivalently the $f_i$) by standard kernel Ridge regression, i.e., by computing
\begin{equation}\label{eq:krr}
\mu_t = \argmin_{h \in \H_k} ~ \sum_{s=1}^t \big(h(i_s, x_s) - y_s\big)^2 + \lambda \|h\|_{\H_k}^2\,,
\end{equation}
where $\H_k$ is the RKHS associated to kernel $k$, and $\lambda > 0$ some regularization parameter.
In the next proposition, we provide some intuition about the role of parameter $b$ by exhibiting regression problems that are equivalent to problem \eqref{eq:krr} when $b=0$ and $b=+\infty$.
We relegate its proof to the end of this section.

\begin{restatable}{prop}{propextreme}\label{prop:extreme_cases}
Solving the multitask kernel Ridge regression problem \eqref{eq:krr} for $b=0$ is equivalent to solve $N$ independent kernel Ridge regressions, using each task's data separately, i.e., we have
\[
\forall\,i \in [N]\,, \qquad \mu_t(i, \cdot\,\vert\, b=0) = \argmin_{h \in \H_{k_\X}} ~ \sum_{s\colon i_s=i} \big(h(x_s) - y_s\big)^2 + \lambda \|h\|_{\H_{k_\X}}^2\,.
\]
On the other hand, solving \eqref{eq:krr} for $b=+\infty$ is equivalent to solve a unique kernel Ridge regression for all the tasks based on the entire dataset, i.e., we have
\[
\forall\,i \in [N]\,, \qquad \mu_t(i, \cdot\,\vert\, b=+\infty) = \argmin_{h \in \H_{k_\X}} ~ \sum_{s=1}^t \big(h(x_s) - y_s\big)^2 + \lambda N \|h\|_{\H_{k_\X}}^2\,.
\]
\end{restatable}
Hence, choosing some $b$ in $(0, +\infty)$ is choosing some tradeoff between these two extreme regimes.
In \Cref{subsec:mt_algo}, see \Cref{thm:main_thm_regrets}, we discuss how to set $b$ with respect to the tasks at hand.
\Cref{prop:extreme_cases} provides another intuition: to ensure constant regularization, parameter $\lambda$ should decrease with $b$.
This key observation was overlooked in \cite{cesa2013gang}, partly explaining our better guarantees, see \Cref{subsec:mt_algo}.
\smallskip

Another way to gain intuition about the setting is to compute the feature map induced by the multitask kernel $k$.
Let $\phi\colon \X \rightarrow \mathbb{R}^d$ be the canonical feature map associated to $k_\X$\footnote{For the sake of the exposition we consider finite dimensional feature maps. Our results naturally extend to infinite dimensional ones.}, and $\psi \colon \X \times  \uptoN \rightarrow \mathbb{R}^{Nd}$ such that $\psi(i,x) = (0, \ldots, \phi(x), \ldots, 0)$, with non-zero entry at block $i$.
Further define $\A = A \otimes I_d \in \mathbb{R}^{Nd \times Nd}$, where $\otimes$ denote the kronecker product.
It is immediate to check that
\begin{equation}\label{eq:psi_tilde}
\tilde{\psi}(i,x) = \A^{-1/2} \, \psi(i, x) = \Big(A^{-1/2}_{1i}\,\phi(x)\,, \ldots, A^{-1/2}_{ii}\,\phi(x)\,, \ldots, A^{-1/2}_{Ni}\,\phi(x)\Big) \in \mathbb{R}^{Nd}
\end{equation}
is a feature map associated to $k$.
Indeed, for any tasks $i, i' \in [N]$, and points $x, x' \in \X$, we have that $\big\langle \tilde{\psi}(i,x), \tilde{\psi}(i',x') \big\rangle = \sum_k A^{-1/2}_{ik} A^{-1/2}_{i'k} \langle \phi(x), \phi(x')\rangle = A^{-1}_{ii'}\,k_\X(x, x') = k\big((i, x), (i', x')\big)$.
Hence, $\tpsi(i, x)$ stores the feature map $\phi(x)$ in each block $j$, weighted by some coefficient $A^{-1/2}_{ji}(b)$, which quantifies how similar tasks $i$ and $j$ are assumed to be.
When $b=0$, only block $i$ receives $\phi(x)$.
When $b=+\infty$, each block receives $\phi(x)/N$.
Denoting abusively $f = (f_1, \ldots, f_N)$, note that

\begin{equation}\label{eq:identity}
\fmt(i, x) = \langle f_i, \phi(x)\rangle = \langle f, \psi(i, x)\rangle = \big\langle \A^{1/2}f, \A^{-1/2}\psi(i, x)\big\rangle = \big\langle \tf, \tpsi(i, x)\big\rangle\,,
\end{equation}

where $\tf = \big(\sum_j A^{1/2}_{1j}\,f_j\,,\, \ldots \,,\, \sum_jA^{1/2}_{Nj}\,f_j\big) \in \H_k$ is a transformed representation of $\fmt$ that will play a key role in the subsequent analysis.

\paragraph{Comparison to the model of~\citep{chowdhury2021no}.}
We note that an online multitask kernel regression setting has also been considered in \cite{chowdhury2021no}.
However, the setting of \cite{chowdhury2021no} considers \emph{multioutput} functions $f^\text{mo}\colon \X \rightarrow \mathbb{R}^{N}$ such that the $i^\text{th}$ output of $f_\text{mo}$ is given by $f_i$.
The authors learn $f^\text{mo}$ by leveraging matrix-valued kernels, a multiple output extension of scalar-valued kernel methods.
Although the functions modeled by kernel $k$ defined in \eqref{eq:mt_kernel} are isomorphic to that associated to the matrix-valued decomposable kernel given by $k_\mathrm{MV}(x, x') = k_\X(x, x')\,A^{-1}$, where we recall that $A^{-1}_{ii'} = k_\T(i, i')$, we highlight that both models are drastically different.
Indeed, while we only observe a single measurement at each time step $t$, namely a noisy version of $f_{i_t}(x_t)$, the model in \citep{chowdhury2021no} assumes that the learner can access $N$ different measurements, that of the $f_i(x_t)$ for all $i \in [N]$.
This key difference makes our model more flexible.
In particular, in~\citep{chowdhury2021no} all tasks must be observed the same number of times, and in addition at the same observation points.


\subsubsection{Proof of Proposition \ref{prop:extreme_cases}}
Recall first that problem \eqref{eq:krr} writes
\[
\mu_t = \argmin_{h \in \H_k} ~ \sum_{s=1}^t \big(h(i_s, x_s) - y_s\big)^2 + \lambda \|h\|_{\H_k}^2\,.
\]
Using identity \eqref{eq:identity}, we obtain
\[
\mu_t = \argmin_{h \in \H_k} ~ \sum_{s=1}^t \Big( \big\langle h, \tpsi(i_s, x_s)\big\rangle - y_s\Big)^2 + \lambda \|h\|_{\H_k}^2\,,
\]
such that standard results for kernel Ridge regression gives that $\mu_t(x) = \big\langle w_t, \tpsi(x)\big\rangle$, with
\begin{equation}\label{eq:krr_solution}
w_t = \bigg(\sum_{s=1}^t \tpsi(i_s, x_s)\tpsi(i_s, x_s)^\top + \lambda I_{Nd}\bigg)^{-1}\bigg(\sum_{s=1}^t y_s\,\tpsi(i_s, x_s)\bigg)\,.
\end{equation}
Now, recall that $\tpsi(i,x) = \A^{-1/2} \, \psi(i, x) = \Big(A^{-1/2}_{1i}\,\phi(x)\,, \ldots, A^{-1/2}_{ii}\,\phi(x)\,, \ldots, A^{-1/2}_{Ni}\,\phi(x)\Big)$, such that
\begin{align}
\tpsi(i,x \,\vert\,b=0) &= (0, \ldots, 0, \underset{\text{block }i}{\phi(x)}, 0, \ldots, 0)\,,\label{eq:feature_map_0}\\
\tpsi(i,x \,\vert\,b=+\infty) &= (\phi(x)/N, \ldots, \phi(x)/N)\,.\label{eq:feature_map_inf}
\end{align}
Substituting \eqref{eq:feature_map_0} in \eqref{eq:krr_solution}, we obtain
\[
w_t = \begin{pmatrix}\displaystyle\sum_{\substack{s \le t\\\text{s.t. } i_s = 1}} \phi(x_s) \phi(x_s)^\top + \lambda I_d &&(0)\\&\ddots&\\(0)&&\displaystyle\sum_{\substack{s \le t\\\text{s.t. } i_s = N}} \phi(x_s) \phi(x_s)^\top + \lambda I_d\end{pmatrix}^{-1} \begin{pmatrix} \displaystyle\sum_{\substack{s \le t\\\text{s.t. } i_s = 1}} y_s \phi(x_s) \\ \vdots \\\displaystyle\sum_{\substack{s \le t\\\text{s.t. } i_s = N}} y_s \phi(x_s)\end{pmatrix}\,,
\]
or again,
\[
\forall\,i\in \uptoN, \qquad w^\mathrm{(i)}_t = \left(\sum_{s \le t \colon i_s=i} \phi(x_s)\phi(x_s)^\top + \lambda I_d\right)^{-1}\left(\sum_{s \le t \colon i_s=i} y_s \phi(x_s)\right)\,,
\]
where $w^{(i)} \in \mathbb{R}^d$ denotes the $i^\text{th}$ block of concatenated vector $w \in \mathbb{R}^{Nd}$.
We thus recover the solutions to the independent regressions stated in the first claim of \cref{prop:extreme_cases}.

Alternatively, substituting \eqref{eq:feature_map_inf} into \eqref{eq:krr_solution}, we obtain
\begin{align*}
w_t &= \begin{pmatrix}\displaystyle\frac{1}{N}\sum_{s=1}^t \phi(x_s) \phi(x_s)^\top + \lambda N\,I_d &&\displaystyle\frac{1}{N}\sum_{s=1}^t \phi(x_s) \phi(x_s)^\top\\&\ddots&\\\displaystyle\frac{1}{N}\sum_{s=1}^t \phi(x_s) \phi(x_s)^\top&&\displaystyle\frac{1}{N}\sum_{s=1}^t \phi(x_s) \phi(x_s)^\top + \lambda N\,I_d\end{pmatrix}^{-1} \begin{pmatrix} \displaystyle\sum_{s=1}^t y_s \phi(x_s) \\ \vdots \\\displaystyle\sum_{s=1}^t y_s \phi(x_s)\end{pmatrix}\\[0.4cm]
&= \begin{pmatrix}\displaystyle\sum_{s=1}^t \phi(x_s) \phi(x_s)^\top + \lambda N\,I_d && (0) \\ & \ddots & \\ (0) && \displaystyle \sum_{s=1}^t \phi(x_s) \phi(x_s)^\top + \lambda N\,I_d\end{pmatrix}^{-1} \begin{pmatrix} \displaystyle\sum_{s=1}^t y_s \phi(x_s) \\ \vdots \\\displaystyle\sum_{s=1}^t y_s \phi(x_s)\end{pmatrix}\,,
\end{align*}
such that for any $i \in \uptoN$ we have $w^\mathrm{(i)}_t = \left(\sum_{s =1}^t \phi(x_s)\phi(x_s)^\top + \lambda N\, I_d\right)^{-1}\left(\sum_{s =1}^t y_s \phi(x_s)\right)$, which are exactly the solutions to the regression problem on the full dataset stated in the second claim of \Cref{prop:extreme_cases}.
Note that the above equality can be easily checked by left multiplying both expressions by $\mathbbm{1}_N\mathbbm{1}_N^\top \otimes \frac{1}{N}\sum_{s=1}^t \phi(x_s) \phi(x_s)^\top + \lambda N\, I_{Nd}$. 
\hfill \qed
%


\subsection{Proof of Theorem~\ref{thm:multitask_conf_intervals}}
\label{apx:interval}

\thminterval*

\begin{proof}
Recall that we are interested in obtaining high probability error bounds of the form
\[
\big|\mu_t(i, x) - \fmt(i, x) \big| \le \beta_t(b) \cdot \sigma_t(i, x \,|\, b)
\]
for a suitable choice of confidence width $\beta_t(b)$.

\subsubsection{A naive confidence width}
Using identity \eqref{eq:identity} and a direct application of \citep[Theorem~3.11 and Remark~3.13]{abbasi2012online}, together with the definition of $\gamma_t^\text{mt}(b)$, we can select
\[
\beta_t = \big\|\tilde{f}\big\| + \lambda^{-1/2} \sqrt{2\big(\gamma_t^\text{mt}(b) + \ln(1/\delta)\big)}\,.
\]
We now upper bound $\big\|\tilde{f}\big\|$ explicitly.
Recall that
\[
\tf = \left(\sum_j A^{1/2}_{1j}\,f_j\,,\, \ldots \,,\, \sum_jA^{1/2}_{Nj}\,f_j\right)\,,
\]
with $A = (1+b)I_N - (b/N) \mathbbm{1}\mathbbm{1}^\top$, so that we have
\begin{align}
\big\|\tf\big\|^2 &= \sum_{i=1}^N \Big\| \sum_{j=1}^N A^{1/2}_{ij}\,f_j \Big\|^2 = \sum_{i=1}^N \sum_{j=1}^N \sum_{k=1}^N A^{1/2}_{ij} A^{1/2}_{ik} \langle f_j, f_k\rangle\nonumber\\
&= \sum_{j, k = 1}^N A_{jk} \langle f_j, f_j \rangle = \sum_{i=1}^N \|f_i\|^2 + b\left(\sum_{i=1}^N \|f_i\|^2 - \frac{1}{N} \sum_{j,k=1}^N \langle f_j, f_k \rangle \right)\,.\label{eq:norm_1}
\end{align}

On the other side, it holds
\begin{align}
\sum_{i=1}^N \bigg\| f_i - \frac{1}{N}\sum_{j=1}^N f_j \bigg\|^2 &= \sum_{i=1}^N\left(\|f_i\|^2 + \frac{1}{N^2} \bigg\|\sum_{j=1}^N f_j\bigg\|^2 - \frac{2}{N} \sum_{j=1}^N \langle f_i, f_j\rangle\right)\nonumber\\
&= \sum_{i=1}^N\|f_i\|^2 + \frac{1}{N} \sum_{i,j=1}^N \langle f_i, f_j \rangle - \frac{2}{N} \sum_{i,j=1}^N \langle f_i, f_j\rangle\nonumber\\
&= \sum_{i=1}^N\|f_i\|^2 - \frac{1}{N} \sum_{i,j=1}^N \langle f_i, f_j \rangle\,.\label{eq:norm_2}
\end{align}

Substituting \eqref{eq:norm_2} into \eqref{eq:norm_1}, we obtain
\[
\big\|\tf\| = \sum_{i=1}^N \|f_i\|^2 + b \sum_{i=1}^N \|f_i - f_\text{avg}\|^2 \le NB^2(1+b\epsilon^2)\,.
\]
Overall, we can thus choose
\begin{equation}\label{eq:beta_naive}
\beta_t^\text{naive}(b)  = B\sqrt{N(1 + b\epsilon^2)} + \lambda^{-1/2}\sqrt{2\big(\gamma_t^\text{mt}(b) + \ln(1/\delta)\big)}\,.
\end{equation}
Note that this choice corresponds also to the confidence intervals utilized in the multitask regression setting of~\cite{cesa2013gang}.
However, we highlight that width \eqref{eq:beta_naive} is \emph{too conservative}.
Indeed, when setting $b = 0$ our model consists in solving single tasks independently, see \Cref{prop:extreme_cases}, such that the confidence width should be of the order $B + \mathcal{O}(\sqrt{\gamma_t^\text{st}})$.
Instead, we have
\[
\beta_t^\text{naive}(0) = \sqrt{N}B + \mathcal{O}\big(\sqrt{\gamma_t^\text{mt}(0)}\,\big) = \sqrt{N}\left(B + \mathcal{O}\big(\sqrt{\gamma_t^\text{st}}\,\big)\right)\,,
\]
which is $\sqrt{N}$ bigger.
%
%
%
%
In the next subsection, we derive an improved confidence width which is tight at $b=0$.

\subsubsection{An improved confidence width}\label{app:improved_width}

As in~\citep[Proof of Theorem 2]{chowdhury2017,abbasi2011improved}, we start by bounding the prediction error as
\begin{align*}
\big|\fmt(i,x)- \mu_t(i,x)\big| &= \big|\fmt(i,x) - \bm{k}_t(i,x)^\top (K_t + \lambda I_t)^{-1}(\bar{\bm{y}}_{1:t} + \bm{\xi}_{1:t})\big| \\
& \leq \underbrace{\big|\fmt(i,x) - \bm{k}_t(i,x)^\top (K_t + \lambda I_t)^{-1}\bar{\bm{y}}_{1:t}\big|}_{\text{bias error}} +  \underbrace{\big|\bm{k}_t(i,x)^\top(K_t + \lambda I_t)^{-1} \bm{\xi}_{1:t}\big|}_{\text{variance error}}\hspace{0.03em},
\end{align*}
where $\bar{\bm{y}}_{1:t} \in \mathbb{R}^t$ is the vector of noise-free outputs such that $\bar{y}_s = \fmt(i_s, x_s) = f_{i_s}(x_s)$, and $\bm{\xi}_{1:t} \in \mathbb{R}^t$ is the vector of noises.
Below, we derive separate bounds for the bias and variance errors.
But before doing so, we depart from~\citep[Proof of Theorem 2]{chowdhury2017} and obtain an explicit lower bound for the predictive variance for task $i$ at point $x$.

\paragraph{\bf Lower bounding the predictive variance.}

First, we introduce some notation. Let $A_\otimes \in \mathbb{R}^{Nd \times Nd}$ such that
\[
A_\otimes \coloneqq A \otimes I_d  = (1+b)I_{Nd} - \frac{b}{N}\mathbbm{1}_N\mathbbm{1}_N^\top\otimes I_d\,,
\]

Moreover, for every $t \in \mathbb{N}$ and $i \in [N]$, let $t_i = \sum_{s \le t} \mathbbm{I}\{i_s = i\}$.
We denote by $\Phi_{t_i} \in \mathbb{R}^{t_i \times d}$ the matrix storing in rows the $\phi(x_s)$ for all $s$ such that $i_s = i$.
Similarly, let $\Psi_t, \tilde{\Psi}_t \in \mathbb{R}^{t \times Nd}$ storing in rows the $\psi(i_s, x_s)$ and $\tilde{\psi}(i_s, x_s)$ respectively, for all $s \le t$.
Recall that $\tpsi(i, x) = \A \psi(i, x)$, such that $\tilde{\Psi}_t = \Psi_t \A$.
Further, define
\begin{equation}\label{eq:def}
    M \coloneqq \Psi_t^\top \Psi_t + \lambda(1+b)I_{Nd} = \begin{bmatrix} M_1 &   &  0  \\ 
 & \ddots & \\ 
0  &    & M_N
\end{bmatrix}, \quad \text{with} \quad M_i = \Phi_{t_i}^\top \Phi_{t_i} + \lambda(1+b)I_d\,.
\end{equation}

By properties of the multitask kernel, see \eqref{eq:mt_kernel}, the predictive variance for task $i$ at decision point $x$ can be lower bounded as:
\begin{align}
    \sigma_t^2(i, x) & = k\big((i,x), (i,x)\big) - \bm{k}_t(i,x)^\top(K_t + \lambda I_t)^{-1}\bm{k}_t(i,x) \nonumber\\
    &= \tpsi(i, x)^\top\left(I_{Nd} - \tilde{\Psi}_t^\top\big(\tilde{\Psi}_t\tilde{\Psi}_t^\top + \lambda I_t\big)^{-1}\tilde{\Psi}_t\right) \tpsi(i, x)\nonumber\\
    &= \tpsi(i, x)^\top\left(I_{Nd} - \big(\tilde{\Psi}_t^\top\tilde{\Psi}_t + \lambda I_{Nd}\big)^{-1}\tilde{\Psi}_t^\top\tilde{\Psi}_t\right) \tpsi(i, x)\label{eq:push_through}\\
    &= \lambda \, \tpsi(i, x)^\top \big(\tilde{\Psi}_t^\top\tilde{\Psi}_t + \lambda I_{Nd}\big)^{-1} \tpsi(i, x)\nonumber\\
    & = \lambda\,\psi(i, x)^\top A_\otimes^{-1/2}\left(A_\otimes^{-1/2} \Psi_t^\top \Psi_t A_\otimes^{-1/2} + \lambda A_\otimes^{-1/2} A_\otimes A_\otimes^{-1/2}\right)^{-1} A_\otimes^{-1/2}\psi(i, x)  \nonumber\\
    & = \lambda\,\psi(i, x)^\top \left(\Psi_t^\top \Psi_t  + \lambda A_\otimes \right)^{-1} \psi(i, x) \nonumber \\ 
    & = \lambda\,\psi(i, x)^\top \left(M  - \frac{\lambda b}{N} \mathbbm{1}_N\mathbbm{1}_N^\top \otimes I_d \right)^{-1} \psi(i, x)\nonumber\\
    & = \lambda\,\psi(i, x)^\top M^{-1} \psi(i, x)\nonumber\\
    &\quad + \lambda\,\psi(i, x)^\top M^{-1} \left(\mathbbm{1}_N\mathbbm{1}_N^\top \otimes \frac{\lambda b}{N}\left(I_{Nd} - \frac{\lambda b}{N} \sum_{i=1}^N M_i^{-1}\right)^{-1}\right)M^{-1}\,\psi(i, x)\label{eq:lemma_sherman}\\
    & = \lambda\,\phi(x)^\top M_i^{-1} \phi(x) + \lambda\, \phi(x)^\top M_i^{-1} \frac{\lambda b}{N}\,\underbrace{\left(I_d - \frac{\lambda b}{N} \sum_{i=1}^N M_i^{-1}\right)^{-1}}_{\coloneqq X^{-1}} M_i^{-1}\phi(x)\nonumber\\
    & = \lambda \left(\big\|M_i^{-1/2} \phi(x)\big\|^2 + \frac{\lambda b}{N}\big\|X^{-1/2} M_i^{-1}\phi(x)\big\|^2\right)\nonumber\\
    & \ge \lambda \big\|M_i^{-1/2} \phi(x)\big\|^2 + \frac{b}{(1+b)N}\big\|\lambda X^{-1} M_i^{-1}\phi(x)\big\|^2\label{eq:var_lb_1}\\
    & \ge (1+b)\big\|\lambda M_i^{-1} \phi(x)\big\|^2 + \frac{b}{(1+b)N}\big\|\lambda X^{-1} M_i^{-1}\phi(x)\big\|^2\,,\label{eq:var_lb_2}
\end{align}
where \eqref{eq:push_through} comes from the \emph{push-through} equality, \eqref{eq:lemma_sherman} from \Cref{lem:SM_kronecker} applied to $D = M$ and $P = - \frac{\lambda b}{N} I_d$, \eqref{eq:var_lb_1} from the fact that $X \succeq 1/(1+b)\,I_d$,  and \eqref{eq:var_lb_2} from $M_i \succeq \lambda(1+b)I_d$.
Indeed, the later can easily be checked from \eqref{eq:def}, which also implies that $(\lambda b)/N\,\sum_i M_i^{-1} \preceq b/(1+b)\,I_d$, such that $X \succeq 1/(1+b)\,I_d$.
We now upper bound the bias error in terms of $\sigma_t^2(i, x)$.

\paragraph{\bf Bounding the bias error.}
Using similar steps as before, we have
\begin{align}
&\big|\fmt(i, x) - \bm{k}_t(i,x)^\top (K_t + \lambda I_t)^{-1}\bar{\bm{y}}_{1:t}\big|\nonumber\\
&= \big|\tpsi(i, x)^\top \tf  - \tpsi(i, x)^\top \tilde{\Psi}_t^\top(\tilde{\Psi}_t\tilde{\Psi}_t^\top + \lambda I_t)^{-1} \tilde{\Psi}_t \tf \big|\nonumber\\
&= \left|\tpsi(i, x)^\top \tf - \tpsi(i, x)^\top (\tilde{\Psi}_t^\top\tilde{\Psi}_t + \lambda I_{Nd})^{-1} \tilde{\Psi}_t^\top \tilde{\Psi}_t \tf\right|\nonumber\\
&= \big|\lambda\,\tpsi(i, x)^\top (\Psi_t^\top\Psi_t + \lambda I_{Nd})^{-1} \tf\big|\nonumber\\
&= \big|\lambda \, 
\psi(i, x)^\top (\Psi_t^\top \Psi_t + \lambda A_\otimes)^{-1} A_\otimes^{1/2} \tf \big|\nonumber\\
&= \left| \lambda\,\psi(i, x)^\top \left[M^{-1} + M^{-1} \left(\mathbbm{1}\mathbbm{1}^\top \otimes \frac{\lambda b}{N} \left(I_d - \sum_{i=1}^N M_i^{-1}\right)^{-1}\right)M^{-1}\right]A_\otimes^{1/2} \tf \,\right|\nonumber\\
&= \left| \big(\lambda M_i^{-1} \phi(x)\big)^\top \left[A_\otimes^{1/2} \tf\right]_{[i]} + \frac{\lambda b}{N} \left(\lambda X^{-1} M_i^{-1} \phi(x)\right)^\top \sum_{l=1}^N \left[M^{-1}A_\otimes^{1/2} \tf \right]_{[l]}\,\right|\nonumber\\
&\le \big\|\lambda M_i^{-1} \phi(x)\big\| \cdot \big\| f_i + b (f_i - f_\text{avg})\big\| + \frac{\lambda b}{N} \left\|\lambda X^{-1} M_i^{-1} \phi(x)\right\| \cdot \left\|\sum_{l=1}^N M^{-1}_l \left[A_\otimes^{1/2} \tf\right]_{[l]} \right\|\nonumber\\
&\le \sqrt{1+b}\,\left\|\lambda M_i^{-1} \phi(x)\right\| \, \frac{B(1 + b\epsilon)}{\sqrt{1+b}}\nonumber\\
&\quad+ \sqrt{\frac{b}{(1+b)N}} \left\|\lambda X^{-1} M_i^{-1} \phi(x)\right\| \cdot \lambda \sqrt{\frac{b(1+b)}{N}} \, \underbrace{\left\| \sum_{l=1}^N M_l^{-1}\left[A_\otimes^{1/2} \tf\right]_{[l]} \right\|}_{:= \|D\|}\nonumber\\
&\le \sqrt{\left((1+b)\left\|\lambda M_i^{-1} \phi(x)\right\|^2 + \frac{b\left\|\lambda X^{-1} M_i^{-1} \phi(x)\right\|^2}{(1+b)N}\right)\left( \frac{B^2(1 + b\epsilon)^2}{1+b} + \frac{\lambda^2 b(1+b)}{N} \left\| D \right\|^2\right)}\label{eq:cs_1}\\
&\le \sigma_t(i, x) \cdot \sqrt{\frac{B^2(1 + b\epsilon)^2}{1+b} + \frac{\lambda^2 b(1+b)}{N} \left\| D \right\|^2}\,,\label{eq:bound_with_D}
\end{align}
where we have used $v_{[i]}$ to denote the block $i$ of a concatenated vector in $\mathbb{R}^{Nd}$, Cauchy-Schwarz inequality to derive \eqref{eq:cs_1}, and lower bound \eqref{eq:var_lb_2} to obtain \eqref{eq:bound_with_D}.
Note that $D$ depends on the products between data matrices $M_l$ and task vectors $f_i, \ldots, f_N$, such that it is unknown in general.
Below, we provide two upper bounds for $\|D\|$.

\paragraph{\it Bound 1 (small $b$ range).}
We can bound $\|D\|$ using Cauchy-Schwarz. We have
\begin{equation*}
    \left\| D \right\| \leq \sum_{l=1}^N \left\|M_l^{-1}\right\|_* \left\|\left[A_\otimes^{1/2} \tf\right]_{[l]} \right\| \leq \frac{1}{\lambda (1+b)} \sum_{l=1}^N \big\| f_l + b(f_l - f_\text{avg})\big\| \leq \frac{NB(1+ b\epsilon)}{\lambda (1+b)}\,,
\end{equation*}
which yields
\begin{equation}\label{eq:bias_small_b}
\big|\fmt(i, x) - k_t(i,x)^\top (K_t + \lambda I_t)^{-1}\bar{y}_{1:t}\big| \le \sigma_t(i, x) \cdot B(1 + b \epsilon) \sqrt{\frac{1 + bN}{1+b}} \,.    
\end{equation}
The above bound is useful when $b$ is small.
Indeed, when $b=0$, according to the above we recover the single-task confidence width $B\cdot\sigma_t(i, x)$.
However, as $b$ goes to $+\infty$, the bound grows as $\mathcal{O}\big(\sqrt{N}b\epsilon\big)$, which is an order of $\sqrt{b}$ faster than the naive one, see \eqref{eq:beta_naive}.
We thus provide another upper bound on $\|D\|$, which is tighter for large values of $b$.

\paragraph{\it Bound 2 (large $b$ range).}
Alternatively, we can bound $\|D\|$ leveraging the SVD of the data matrices used to build $M_l$ (recall that $M_l = \lambda (1+b)I_d + \Phi_{t_l}^\top \Phi_{t_l}$).
For $l \le N$, let $\Phi_{t_l}^\top \Phi_{t_l} = \sum_k \sigma_k^{(l)} u_k^{(l)}{u_k^{(l)}}^\top$ be the SVD of the data matrix $\Phi_{t_l}^\top \Phi_{t_l}$.
We have
\[
M_l^{-1} = \sum_k \frac{1}{\sigma_k^{(l)} + \lambda(1+b)} u_k^{(l)}{u_k^{(l)}}^\top = \frac{1}{\lambda(1+b)} I_d - \sum_k \frac{\sigma_k^{(l)}}{\lambda(1+b)\big(\sigma_k^{(l)} + \lambda(1+b)\big)} u_k^{(l)}{u_k^{(l)}}^\top\,,
\]
so that
\begin{align}
\|D\| & = \bigg\|\sum_{l=1}^N M_l^{-1} \big(f_l + b(f_l - f_\text{avg})\big) \bigg\| \nonumber\\
&\le \frac{1}{\lambda(1+b)} \left\| \sum_{l=1}^N f_l \right\| +  \frac{1}{\lambda(1+b)} \left\| \sum_{k, l} \frac{\sigma_k^{(l)}}{\sigma_k^{(l)} + \lambda(1+b)} u_k^{(l)}{u_k^{(l)}}^\top \big(f_l +  b(f_l - f_\text{avg})\big) \right\|\nonumber\\
&\le \frac{NB}{\lambda(1+b)} + \frac{1}{\lambda(1+b)} \sum_{k, l} \frac{\sigma_k^{(l)}}{\sigma_k^{(l)} + \lambda(1+b)} \big\| f_l +  b(f_l - f_\text{avg})\big\|\nonumber\\
&\le \frac{NB}{\lambda(1+b)} + \frac{B(1+b\epsilon)}{\lambda(1+b)} \sum_{l=1}^N \Tr\left(K_{t_l} \big(K_{t_l} + \lambda(1+b) I_{t_l}\big)^{-1}\right)\label{eq:data_dep}\\
&\le \frac{NB}{\lambda(1+b)} + \frac{B(1+b\epsilon)}{\lambda(1+b)} \sum_{l=1}^N \Tr(K_{t_l}) \cdot \lambda_\text{max}\left(\big(K_{t_l} + \lambda(1+b) I_{t_l}\big)^{-1}\right)\nonumber\\
&\le  \frac{NB}{\lambda(1+b)} +  \frac{B (1 + b\epsilon)}{\lambda^2(1+b)^2} \sum_{l=1}^N t_l\nonumber\\
&\le  \frac{NB}{\lambda(1+b)} +  \frac{B (1 + b\epsilon)}{\lambda^2(1+b)^2}~t\nonumber\,.
\end{align}

We note that in practice the data-dependent bound~\eqref{eq:data_dep} might be tighter than the latter one when $\lambda$ is small. However, for simplicity of the exposition we focus on the latter data-independent bound.
Substituting it into \eqref{eq:bound_with_D}, we obtain
\begin{align}
\big|\fmt(i,x) - \bm{k}_t(i, x)^\top (K_t + &\lambda I_t)^{-1}\bar{\bm{y}}_{1:t}|\nonumber\\
&\le \sigma_t(i, x) \cdot \sqrt{\frac{B^2(1+b\epsilon)^2}{1+b} + 2B^2 \frac{bN}{1+b} + 2t^2 \frac{B^2(1+b\epsilon)^2b}{\lambda^2(1+b)^3N}}\nonumber\\
&= \sigma_t(i, x) \cdot B \sqrt{\frac{(1+b\epsilon)^2}{1+b} + \frac{2bN}{1+b} + \frac{2b(1+b\epsilon)^2}{N\lambda^2(1+b)^3}\,t^2}\,.\label{eq:bias_big_b}
\end{align}
When $b$ goes to $0$, we recover $B \cdot\sigma_t(i, x)$, as for Bound 1.
However, the above bound is more useful when $b$ is large.
Indeed, when $b$ goes to $+\infty$, we obtain $B \sqrt{b \epsilon^2 + 2N + 2\epsilon^2 t^2/N\lambda^2\,}\,\sigma_t(i, x) = \mathcal{O}\big(B \sqrt{b}\epsilon\big)\,\sigma_t(i, x)$, which improves by a factor $\sqrt{N}$ over the $\mathcal{O}\big(B \sqrt{bN}\epsilon\big)\,\sigma_t(i, x)$ term obtained with the naive bound.
Recall that when $b$ goes to $+\infty$, MT regression is equivalent to solve a single averaged task based on the whole dataset, see \Cref{prop:extreme_cases}, such that obtaining a confidence width independent from $N$ is expected.
Overall, combining \eqref{eq:beta_naive}, \eqref{eq:bias_small_b} and \eqref{eq:bias_big_b}, we obtain that
\begin{equation}\label{eq:bias_final}
 \big|\fmt(i, x) - \bm{k}_t(i, x)^\top (K_t + \lambda I_t)^{-1}\bar{\bm{y}}_{1:t}\big|  \leq \beta_t^\textnormal{bias}(b) \cdot \sigma_t(i, x)\,,
\end{equation}
where
$\beta_t^\textnormal{bias}(b) = B\,\min\left\{\sqrt{N(1 + b\epsilon^2)},\, (1+b\epsilon)\sqrt{\frac{1+bN}{1+b}},\, \sqrt{\frac{(1+b\epsilon)^2}{1+b} + \frac{2bN}{1+b} + \frac{2b(1+b\epsilon)^2}{N\lambda^2(1+b)^3}\,t^2} \,\right\}$.
We now turn to the variance error.

\paragraph{\bf Bounding the variance error.}
Using \eqref{eq:beta_naive}, the variance error can be naively bounded by $\beta_t^\text{naive, var}(b) \cdot \sigma_t(i, x)$, where $\beta_t^\text{naive, var}(b) = \lambda^{-1/2} \sqrt{2\big(\gamma^\text{mt}_t(b) + \ln(1/\delta)\big)}$.
However, note that the above bound is conservative when parameter $b$ is small.
Indeed, in the limit of $b=0$ (tasks are treated independently), we know that such error should only depend on the information gain of task $i$, i.e., $\gamma^\text{st}_{t_i}$.
Instead, $\gamma_t^\text{mt}(0) = \mathcal{O}\big(N \, \gamma^\text{st}_{t_i}\big)$, which is $N$ times bigger.
Let $P_i := \Phi_{t_i}^\top \Phi_{t_i} + I_d$, and note that $M_i^{-1} \preceq P_i^{-1}$, since $\lambda \ge 1/(1+b)$.
Looking at a single task $i$, with probability $1-\delta$ we have
\begin{align}
\big\|M_i^{-1/2} \Phi_{t_i}^\top \bm{\xi}_{[t_i]} \big\|^2 &\le \big\|P_i^{-1/2} \Phi_{t_i}^\top \bm{\xi}_{[t_i]} \big\|^2\nonumber\\
&= \bm{\xi}_{[t_i]}^\top \Phi_{t_i} P_i^{-1} \Phi_{t_i}^\top \bm{\xi}_{[t_i]}\nonumber\\
&= \bm{\xi}_{[t_i]}^\top \Phi_{t_i} \big(\Phi_{t_i}^\top \Phi_{t_i} + I_d\big)^{-1} \Phi_{t_i}^\top \bm{\xi}_{[t_i]}\nonumber\\
&= \bm{\xi}_{[t_i]}^\top \big(K_{t_i} + I_{t_i}\big)^{-1} K_{t_i}\,\bm{\xi}_{[t_i]}\nonumber\\
&= \bm{\xi}_{[t_i]}^\top \big(I_{t_i} + K^{-1}_{t_i}\big)^{-1} \bm{\xi}_{[t_i]}\nonumber\\
&\le 2\left(\frac{1}{2} \ln \big|I_{t_i} + K_{t_i}\big| + \ln(1/\delta)\right)\label{eq:norm_residuals}\\
&\le 2\left(\frac{1}{2} \ln \big|I_{t_i} + \lambda^{-1}K_{t_i}\big| + \ln(1/\delta)\right)\label{eq:lambda_1}\\
&\leq 2\big(\gamma^\text{st}_{t_i} + \ln(1/\delta)\big)\,,\nonumber
\end{align}
where $\bm{\xi}_{[t_i]} \in \mathbb{R}^{t_i}$ contains the observation noises related to the time steps when task $i$ was active, and $K_{t_i} \in \mathbb{R}^{t_i \times t_i}$ is the individual Gram matrix based on such observations.
Equation \eqref{eq:norm_residuals} is obtained by applying \cite[Theorem~1]{chowdhury2017} with $\eta=0$, while \eqref{eq:lambda_1} derives from $\lambda \le 1$.
By the union bound, we get that with probability at least $1-\delta$, we have
\[
\sup_{i \le N} \big\|P_i^{-1/2} \Phi_{t_i}^\top \bm{\xi}_{[t_i]} \big\|^2 \le 2\big(\gamma^\text{st}_t + \ln(N/\delta)\big)\,.
\]
Then, we obtain
\begin{align}
\big|k_t(i,x)^\top&(K_t + \lambda I_t)^{-1} \bm{\xi}_{1:t}\big|\nonumber\\
&= \Big|\tilde{\psi}(i, x)^\top \big(\tilde{\Psi}_t^\top \tilde{\Psi}_t + \lambda I_{Nd}\big)^{-1} \tilde{\Psi}_t^\top \bm{\xi}_{1:t}\Big|\nonumber\\
&= \Big|\psi(i, x)^\top \big(\Psi_t^\top\Psi_t + \lambda A_\otimes\big)^{-1}  \Psi_t^\top \bm{\xi}_{1:t}\Big|\nonumber\\ 
&= \left|\phi(x)^\top M_i^{-1} \Phi_{t_i}^\top \bm{\xi}_{[t_i]} + \frac{\lambda b}{N} \big(X^{-1}M_i^{-1}\phi(x)\big)^\top \sum_{l=1}^N \big[M^{-1} \Psi_t^\top \bm{\xi}_{1:t}\big]_{[l]} \right|\nonumber\\
&\leq \sqrt{\lambda}\,\Big\|M_i^{-1/2}\phi(x)\Big\| \cdot \frac{1}{\sqrt{\lambda}}\Big\|M_i^{-1/2} \Phi_{t_i}^\top \bm{\xi}_{[t_i]}\Big\|\nonumber\\
&\quad+ \sqrt{\frac{b}{(1+b)N}}\,\Big\|\lambda X^{-1}M_i^{-1}\phi(x)\Big\| \cdot \sqrt{\frac{b(1+b)}{N}} \left\| \sum_{l=1}^N M_l^{-1} \Phi_{t_l}^\top \bm{\xi}_{[t_l]} \right\|\nonumber\\ 
&\leq \sqrt{\lambda}\,\Big\|M_i^{-1/2}\phi(x)\Big\| \cdot \frac{1}{\sqrt{\lambda}}\Big\|M_i^{-1/2} \Phi_{t_i}^\top \bm{\xi}_{[t_i]}\Big\|\nonumber\\
&\quad+ \sqrt{\frac{b}{(1+b)N}}\,\Big\|\lambda X^{-1}M_i^{-1}\phi(x)\Big\| \cdot \sqrt{b(1+b)N} ~ \sup_{l \le N} \big\|M_l^{-1/2}\big\|_* \cdot \big\|M_l^{-1/2} \Phi_{t_l}^\top \bm{\xi}_{[t_l]} \big\|\nonumber\\
&\leq \sqrt{\lambda}\,\Big\|M_i^{-1/2}\phi(x)\Big\| \cdot \frac{1}{\sqrt{\lambda}}\Big\|M_i^{-1/2} \Phi_{t_i}^\top \bm{\xi}_{[t_i]}\Big\|\nonumber\\
&\quad+ \sqrt{\frac{b}{(1+b)N}}\,\Big\|\lambda X^{-1}M_i^{-1}\phi(x)\Big\| \cdot \sqrt{\frac{bN}{\lambda}}\,\sup_{l \le N} \big\|M_l^{-1/2} \Phi_{t_l}^\top \bm{\xi}_{[t_l]} \big\|\nonumber\\
&\leq \left(\lambda \Big\|M_i^{-1/2}\phi(x) \Big\|^2  + \frac{b}{(1+b)N}\Big\|\lambda X^{-1}M_i^{-1}\phi(x)\Big\|^2 \right)^{1/2}\nonumber\\
&\qquad\cdot\left(\frac{1}{\lambda} \Big\|P_i^{-1/2} \Phi_{t_i}^\top \bm{\xi}_{[t_i]}\Big\|^2 + \frac{bN}{\lambda} \, \sup_{l \le N} \big\|P_l^{-1/2} \Phi_{t_l}^\top \bm{\xi}_{[t_l]} \big\|^2\right)^{1/2}\nonumber\\
& \leq \lambda^{-1/2}\,\sigma_t(i,x)~\sup_{l \le N} \big\|P_l^{-1/2} \Phi_{t_l}^\top \bm{\xi}_{[t_l]} \big\|\, \sqrt{1 + bN}\label{eq:using}\\
& \leq \lambda^{-1/2}\sqrt{2(1 + bN)\big(\gamma^\text{st}_t + \ln(N/\delta)\big)}\,\sigma_t(i,x)\,,\nonumber
\end{align}
where \eqref{eq:using} comes from lower bound \eqref{eq:var_lb_1}.
Finally, we can take the minimum of this bound and the naive one.
Using the union bound again, with probability at least $1-2\delta$, we have
\begin{equation}\label{eq:var_final}
\big|\bm{k}_t(i,x)^\top(K_t + \lambda I_t)^{-1} \bm{\xi}_{1:t}\big|  \leq \beta_t^\text{var}(b) \cdot \sigma_t(i,x),
\end{equation}
where $\beta_t^\text{var}(b) = \lambda^{-1/2} \min \left\{\sqrt{2\big(\gamma^\text{mt}_t(b) + \ln(1/\delta)\big)}, \sqrt{2(1+bN)\big(\gamma^\text{st}_t + \ln(N/\delta)\big)}\,\right\}.$

\paragraph{Overall error bound.}
We can obtain the overall prediction error bound by combining bounds \eqref{eq:bias_final} and \eqref{eq:var_final} for the bias and variance errors respectively.
Hence, with probability $1-2\delta$, we have
\begin{align*}
\big|\fmt(i,x)- \mu_t(i,x)\big| & = \big|\fmt(i,x) - \bm{k}_t(i,x)^\top (K_t + \lambda I_t)^{-1}(\bar{\bm{y}}_{1:t} + \bm{\xi}_{1:t})\big| \\
&\leq \underbrace{\big|\fmt(i,x) - \bm{k}_t(i,x)^\top (K_t + \lambda I_t)^{-1}\bar{\bm{y}}_{1:t}\big|}_{\text{bias error}} +  \underbrace{\big|\bm{k}_t(i,x)^\top(K_t + \lambda I_t)^{-1} \bm{\xi}_{1:t}\big|}_{\text{variance error}}\\ 
&\leq \big(\beta_t^\textnormal{bias}(b) + \beta_t^\text{var}(b)\big) \cdot \sigma_t(i,x)\,,
\end{align*}
with $\beta_t^\textnormal{bias}(b) = B \min\left\{\sqrt{N(1 + b\epsilon^2)}, (1+b\epsilon)\sqrt{\frac{1+bN}{1+b}}, \sqrt{\frac{(1+b\epsilon)^2}{1+b} +  \frac{2bN}{1+b} + \frac{2b(1+b\epsilon)^2}{N\lambda^2(1+b)^3}\,t^2}\,\right\}$,
and
$\beta_t^\text{var}(b) = \lambda^{-1/2} \min \left\{\sqrt{2\big(\gamma^\text{mt}_t(b) + \ln(1/\delta)\big)}, \sqrt{2(1+bN)\big(\gamma^\text{st}_t + \ln(N/\delta)\big)} \right\}$.
In particular, we obtain \Cref{thm:multitask_conf_intervals} by considering specific combinations among the minimums involved in the definitions of $\beta_t^\textnormal{bias}(b)$ and $\beta_t^\textnormal{var}(b)$.
The resulting $\beta_t^\textnormal{small-$b$}(b)$ and $\beta_t^\textnormal{large-$b$}(b)$ are chosen to be tight when $b$ goes to $0$ or $+\infty$ respectively.
\end{proof}

\subsubsection{A Kronecker Sherman-Morrison Lemma}

We now provide a lemma which extends the Sherman-Morrison formula to kronecker matrices.
\begin{lem}\label{lem:SM_kronecker}
Let $D_1, \ldots, D_N \in \mathbb{R}^{d \times d}$ be invertible, $D = \mathrm{diag}\big(D_1, \ldots, D_N\big) \in \mathbb{R}^{Nd \times Nd}$, and $P \in \mathbb{R}^{d \times d}$ that commutes with $D_i$ for all $i \in [N]$.
Then we have
\[
\left(D + \mathbbm{1}_N\mathbbm{1}_N^\top \otimes P\right)^{-1} = D^{-1} + D^{-1} \left(\mathbbm{1}_N\mathbbm{1}_N^\top \otimes Q\right)D^{-1}\,,
\]
where $Q = - \Big(I_d + P\big(D_1^{-1} + \ldots + D_N^{-1}\big)\Big)^{-1}P = - P\Big(I_d + P\big(D_1^{-1} + \ldots + D_N^{-1}\big)\Big)^{-1}$.
\end{lem}

\begin{proof}
It is immediate to check that $\left(D + \mathbbm{1}_N\mathbbm{1}_N^\top \otimes P\right) \left(D^{-1} + D^{-1} \left(\mathbbm{1}_N\mathbbm{1}_N^\top \otimes Q\right)D^{-1}\right) = \left(D^{-1} + D^{-1} \left(\mathbbm{1}_N\mathbbm{1}_N^\top \otimes Q\right)D^{-1}\right)\left(D + \mathbbm{1}_N\mathbbm{1}_N^\top \otimes P\right) = I_{Nd}$.
\end{proof}

\section{New Guarantees for Online Learning}\label{app:online_learning}
\par{\bf Details on the independent regret bound.}
We analyze the regret of the strategy which runs $N$ independent instances of IGP-UCB \citep{chowdhury2017}, one per task.
Choosing the single-task confidence width $\beta^\text{st}_t = B + \sqrt{2\big(\gamma^\text{st}_t + \ln(N/\delta)\big)}$, one can control the individual task regrets and obtain with probability at least $1-\delta$:
\begin{align*}
\Rmt(T) &= \sum_{i=1}^N ~ \sum_{t \colon i_t = i} \max_{x \in \X} \fmt(i, x) - \fmt(i, x_t)\\
&\le 4 \sum_{i=1}^N \beta^\text{st}_{T_i}\sqrt{T_i\,\gamma^\text{st}_{T_i}}\\
&\le 6\left( B\sqrt{NT\gamma^\text{st}_T} + \sqrt{NT\gamma^\text{st}_T\big(\gamma^\text{st}_T+\ln(N/\delta)\big)}\right)\,,
\end{align*}
where we have used that $\gamma^\text{st}_{T_i} \le \gamma^\text{st}_T$ and Jensen's inequality.
We exactly recover the first bound in \Cref{thm:main_thm_regrets}.

\subsection{Proof of Lemma \ref{lem:generic_bound}}
\label{apx:jensen}

\lemgeneric*

\begin{proof}
The proof follows from standard arguments, see e.g., \cite[Theorem~1]{srinivas2010gaussian} and \cite[Theorem~3]{chowdhury2017}, reproduced here for completeness.
For, $i \in [N]$, let $x^*_i = \argmax_{x \in \X} \fmt(i, x)$.
With probability $1-\delta$, we have
\begin{align}
\sum_{t=1}^T ~ \max_{x \in \X} \fmt(i_t, x) &- \fmt(i_t, x_t)\nonumber\\
&\le \sum_{t=1}^T \text{ucb}_{t-1}(i_t, x^*_{i_t}\,|\,b) - \text{ucb}_{t-1}(i_t, x_t\,|\,b) + 2 \beta_t(b) \cdot \sigma_{t-1}(i_t, x_t\,|\,b)\nonumber\\
&\le 2 \sum_{t=1}^T \beta_t(b) \cdot \sigma_{t-1}(i_t, x_t\,|\,b)\nonumber\\
&\le 2 \beta_T(b) \sqrt{T \, \sum_{t=1}^T \sigma^2_{t-1}(i_t, x_t\,|\,b)}~.\label{eq:1}
\end{align}
The main twist is that our regularization parameter $\lambda$ might be smaller than $1$, preventing from a direct adaptation of \cite[Lemma~4]{chowdhury2017} to bound $\sum_{t=1}^T \sigma^2_{t-1}(i_t, x_t\,|\,b)$.
However, this happens not to be a problem since our multitask predictive variance are also smaller than in the single-task case.
Indeed, as long as $\lambda \ge (N + b)/(N + bN) = A(b)^{-1}_{ii}$, we have for all $i$ and $x$
\begin{align*}
\sigma^2_t(i, x) &= k\big((i, x), (i, x)\big) - \bm{k}_t(i, x)^\top(K_t + \lambda I_t)^{-1} \bm{k}_t(i, x)\\
&\le k\big((i, x), (i, x)\big)\\
&= A^{-1}_{ii}~k_\mathcal{X}(x, x)\\
&\le \lambda\,.
\end{align*}
Therefore, we have
\begin{equation}\label{eq:2}
\sum_{t=1}^T \sigma_{t-1}^2(i_t, x_t) = \lambda \sum_{t=1}^T \lambda^{-1}\sigma_{t-1}^2(i_t, x_t) \le 2\lambda \sum_{t=1}^T \ln\big(1 + \lambda^{-1}\sigma_{t-1}^2(i_t, x_t)\big) \le 4\lambda\,\gamma^\text{mt}_T(b)\,,
\end{equation}
where we have used that $x \le \ln(1+x)$ for any $x \in [0, 1]$, applied to the $\lambda^{-1}\sigma_{t-1}^2(i_t, x_t) \le 1$, and \citep[Lemma~3]{chowdhury2017}.
Substituting \eqref{eq:2} into \eqref{eq:1} concludes the proof.
\end{proof}

\subsection{Proof of Proposition \ref{prop:bound_info_gain}}

\propinfogain*

\begin{proof}
Recall that the multitask kernel writes $k\big((i,x), (i',x')\big) = k_\T(i,i') \cdot k_\X(x, x')$.
Hence, the multitask Gram matrix $K_T$ can be written as $K_T = K_\T \odot K_\X$, where $K_\T, K_\X \in \mathbb{R}^{T \times T}$ are the task (respectively domain) Gram matrices.
Moreover, up to rearranging the rows and columns of $K_\T$ (which does not change the determinant), we can assume that points are ordered by task activations.
Let $T_i = \sum_{t=1}^T \mathbbm{I}\{i_t = i\}$ be the number of times task $i$ has been queried.
We have
\begin{align*}
K_\T &= \begin{pmatrix}\frac{b+N}{(1+b)N}~\mathbbm{1}_{T_1}\mathbbm{1}_{T_1}^\top && \frac{b}{(1+b)N}~\mathbbm{1}_{T_1}\mathbbm{1}_{T - T_1}^\top \\ & \ddots & \\ \frac{b}{(1+b)N}~\mathbbm{1}_{T_N}\mathbbm{1}_{T - T_N}^\top && \frac{b+N}{(1+b)N}~\mathbbm{1}_{T_N}\mathbbm{1}_{T_N}^\top\end{pmatrix}\\[0.5cm]
&= \frac{b}{1+b}~\frac{\mathbbm{1}_{T}\mathbbm{1}_{T}^\top}{N} + \frac{1}{1+b} \begin{pmatrix}\mathbbm{1}_{T_1}\mathbbm{1}_{T_1}^\top &&0\\ & \ddots & &\\ 0 &&\mathbbm{1}_{T_N}\mathbbm{1}_{T_N}^\top\end{pmatrix}\,.
\end{align*}
We also introduce the block notation $K_\X^{(i, j)} \in \mathbb{R}^{T_i \times T_j}$ and $K_\X^\text{diag} \in \mathbb{R}^{T \times T}$ such that
\[
K_\X = \begin{pmatrix}&&\\&K_\X^{(i, j)}&\\&&\end{pmatrix}\,, \qquad \text{and} \qquad K_\X^\text{diag} = \begin{pmatrix}K_\X^{(1, 1)}&&0\\&\ddots&\\0&&K_\X^{(N, N)}\end{pmatrix}\,.
\]
Our bounds are based on the observation that $M \mapsto \ln|M|$ is a concave function, and therefore upper bounded by its first order Taylor approximation.
In other words, for any positive semi-definite matrices $X$ and $Y$ we have
\begin{equation}\label{eq:logdet}
\ln|X| \le \ln|Y| + \Tr\big(Y^{-1}(X-Y)\big)\,.
\end{equation}
Hence, for any $ \ge 0$ we have
\begin{align}
2\gamma^\text{mt}_T(b) &= \ln\big|I_T + \lambda^{-1}K_T\big|\nonumber\\
&= \ln\big|I_T + \lambda^{-1}K_\T \odot K_\X\big|\nonumber\\
&= \ln\left|I_T + \lambda^{-1}\left(\frac{b}{1+b}\frac{K_\X}{N} + \frac{1}{1+b}K_\X^\text{diag}\right)\right|\nonumber\\
&= \ln\left|\frac{b}{1+b}\left(I_T + \lambda^{-1}\frac{K_\X}{N}\right) + \frac{1}{1+b}\left(I_T + \lambda^{-1} K_\X^\text{diag}\right)\right|\nonumber\\
&\le \ln\left|\frac{1}{1+b}\left(I_T + \lambda^{-1} K_\X^\text{diag}\right)\right|\nonumber\\
&\quad+ \Tr\left((1+b)\left(I_T + \lambda^{-1} K_\X^\text{diag}\right)^{-1}\frac{b}{1+b}\left(I_T + \lambda^{-1}\frac{K_\X}{N}\right)\right)\label{eq:log_det_cvx1}\\
&= \ln\left|I_T + \lambda^{-1} K_\X^\text{diag}\right| -T\ln(1+b) + b\,\Tr\left(\left(I_T + \lambda^{-1} K_\X^\text{diag}\right)^{-1}\left(I_T + \lambda^{-1}\frac{K_\X}{N}\right)\right)\nonumber\\
&\le 2 \sum_{i=1}^N \gamma^\text{st}_{T_i} + b\,\Tr\left(\left(I_T + \lambda^{-1} K_\X^\text{diag}\right)^{-1}\left(I_T + \lambda^{-1}\frac{K_\X}{N}\right)\right) - T\ln(1+b)\,,\label{eq:bound_logdet1}
\end{align}
where \eqref{eq:log_det_cvx1} comes from \eqref{eq:logdet} applied with $X = \frac{b}{1+b}\left(I_T + \lambda^{-1}\frac{K_\X}{N}\right) + \frac{1}{1+b}\left(I_T + \lambda^{-1} K_\X^\text{diag}\right)$ and $Y=\frac{1}{1+b}\left(I_T + \lambda^{-1} K_\X^\text{diag}\right)$.
We now take a closer look on the second term.
We have
\begin{align}
\Tr\bigg(\Big(I_T &+ \lambda^{-1} K_\X^\text{diag}\Big)^{-1}\left(I_T + \lambda^{-1}\frac{K_\X}{N}\right)\bigg)\nonumber\\
&= \Tr\left(\begin{pmatrix}\Big(I_{T_1} + \lambda^{-1} K_\X^{(1, 1)}\Big)^{-1}&&0\\&\ddots&\\0&&\Big(I_{T_N} + \lambda^{-1} K_\X^{(N, N)}\Big)^{-1}\end{pmatrix}\left(I_T + \lambda^{-1}\frac{K_\X}{N}\right)\right)\nonumber\\
&= \sum_{i=1}^N \Tr\left(\Big(I_{T_i} + \lambda^{-1} K_\X^{(i, i)}\Big)^{-1}\left(I_{T_i} + \lambda^{-1}\frac{K_\X^{(i, i)}}{N}\right)\right)\nonumber\\
&= \sum_{i=1}^N \sum_{\tau=1}^{T_i} \frac{\lambda + \sigma^{(i)}_\tau/N}{\lambda + \sigma^{(i)}_\tau}\label{eq:eigen_values}
\end{align}
where $\big\{\sigma_\tau^{(i)}\big\}_{\tau \le T_i}$ are the eigenvalues of $K_\X^{(i, i)}$, possibly equal to $0$.
For any $i$, let $F^{(i)}\colon \mathbb{R}^{T_i} \rightarrow \mathbb{R}$ the functions which to any $\bm{\sigma} = (\sigma_1, \ldots, \sigma_{T_i})$ associates $F^{(i)}(\bm{\sigma}) = \sum_{\tau=1}^{T_i} (\lambda + \sigma_\tau/N)/(\lambda + \sigma_\tau)$.
For any $\tau_0, \tau_1$, we have
\begin{align*}
(\sigma_{\tau_0} - &\sigma_{\tau_1})\left(\frac{\partial F^{(i)}(\bm{\sigma})}{\partial \sigma_{\tau_0}} - \frac{\partial F^{(i)}(\bm{\sigma})}{\partial \sigma_{\tau_1}}\right)\\
&= (\sigma_{\tau_0} - \sigma_{\tau_1})\left(\frac{(1/N)(\lambda + \sigma_{\tau_0}) - (\lambda + \sigma_{\tau_0}/N)}{(\lambda + \sigma_{\tau_0})^2} - \frac{(1/N)(\lambda + \sigma_{\tau_1}) - (\lambda + \sigma_{\tau_1}/N)}{(\lambda + \sigma_{\tau_1})^2}\right)\\
&= \lambda\,\frac{N-1}{N}(\sigma_{\tau_1} - \sigma_{\tau_0})\left(\frac{1}{(\lambda + \sigma_{\tau_0})^2} - \frac{1}{(\lambda + \sigma_{\tau_1})^2}\right)\\
&\ge 0\,,
\end{align*}
such that $F^{(i)}$ is Schur-convex.
Hence, \eqref{eq:eigen_values} is maximized at eigenvalues of the form $(T_i, 0, \ldots, 0)$, since the latter majorizes any other admissible distribution of the eigenvalues (recall that we must have $\sigma_\tau^{(i)} \ge 0$ and $\sum_{\tau=1}^{T_i} \sigma^{(i)}_\tau = \Tr\big(K_\X^{(i, i)}\big) \le T_i$\,), with value
\[
\sum_{i=1}^N \, (T_i - 1) + \frac{\lambda + T_i/N}{\lambda + T_i} \le T - N + N \, \frac{1+1/N}{2} \le T - \frac{N}{4}\,,
\]
where we have used that $\lambda \le 1$, $T_i \ge 1$, and $N \ge 2$.
Substituting into \eqref{eq:bound_logdet1}, we finally obtain
\begin{equation}\label{eq:bound_logdet}
\gamma_T^\text{mt}(b) \le N \gamma_T^\text{st} + \frac{b}{2}\left(T - \frac{N}{4}\right) - \frac{T}{2}\ln(1+b)\,.
\end{equation}
\bigskip

The second bound is obtained by modifying \eqref{eq:log_det_cvx1}.
Instead, we now apply \eqref{eq:logdet} with $X = \frac{b}{1+b}\left(I_T + \lambda^{-1}\frac{K_\X}{N}\right) + \frac{1}{1+b}\left(I_T + \lambda^{-1} K_\X^\text{diag}\right)$ and $Y=\frac{b}{1+b}\left(I_T + \lambda^{-1}\frac{K_\X}{N}\right)$.
We obtain
\begin{align}
2\gamma^\text{mt}_T(b) & \le \ln\left|\frac{b}{1+b}\left(I_T + \lambda^{-1}\frac{K_\X}{N}\right)\right|\nonumber\\
&\quad+ \Tr\left(\frac{1+b}{b}\left(I_T + \lambda^{-1}\frac{K_\X}{N}\right)^{-1}\frac{1}{1+b}\Big(I_T + \lambda^{-1}{K_\X}^\text{diag}\Big)\right)\nonumber\\
&=\ln\left|I_T + \lambda^{-1}\frac{K_\X}{N}\right| - T\ln\left(1+\frac{1}{b}\right) + \frac{1}{b}\,\Tr\left(\left(I_T + \lambda^{-1}\frac{K_\X}{N}\right)^{-1}\left(I_T + \lambda^{-1}K_\X^\text{diag}\right)\right)\nonumber\\
&\le 2\,\gamma_T^\text{st} + \frac{1}{b} \Tr\left(I_T + \lambda^{-1}K_\X^\text{diag}\right)\label{eq:von_n}\\
&= 2\,\gamma_T^\text{st} + \frac{(1 + \lambda^{-1})T}{b}\nonumber\\
&\le 2\,\gamma_T^\text{st} + 2\,\frac{T}{\lambda b}\,,\nonumber
\end{align}
where \eqref{eq:von_n} comes from von Neumann's trace inequality and is tight when $N \rightarrow +\infty$.
\end{proof}

\subsection{Proof of Theorem \ref{thm:main_thm_regrets}}

\thmmain*

\begin{proof}
From \Cref{lem:generic_bound} and the choice $\beta_t = \beta^\text{new}_t$, we have with probability at least $1-2\delta$
\[
\Rmt(T) \le 4\,\beta^\text{new}_T(b) \sqrt{\lambda\,T\, \gamma_T^\textnormal{mt}(b)}\,.
\]

\paragraph{First bound, recovering independent learning.}
Hence, in particular, we have
\begin{align*}
\Rmt(T) &\le 4\,\beta^\text{small-$b$}_T(b) \sqrt{\lambda\,T\, \gamma_T^\textnormal{mt}(b)}\\
&= 4\left(B(1+b\epsilon)\sqrt{\frac{1+bN}{1+b}} + \lambda^{-1/2} \sqrt{2(1+bN)\big(\gamma_T^\textnormal{st} + \ln(N/\delta)\big)}\right) \sqrt{\lambda\,T\, \gamma_T^\textnormal{mt}(b)}\\
&\le 6\left(B\,\frac{1+b\epsilon}{1+b}\sqrt{\frac{(1+bN)(b+N)}{N}} + \sqrt{(1+bN)\big(\gamma_T^\textnormal{st} + \ln(N/\delta)\big)}\right) \sqrt{T\, \big(N\gamma^\text{st}_T + bT\big)}\,,
\end{align*}
where we have used the first claim of \Cref{prop:bound_info_gain} and the choice $\lambda = (N + b)/(N+bN)$.
Substituting $b=0$ in the above equation, we recover the independent learning bound, i.e.,
\begin{equation}\label{eq:bound_b0}
\Rmt(T) \le 6\left(B\sqrt{NT\gamma^\text{st}_T} + \sqrt{NT\gamma_T^\textnormal{st}\big(\gamma_T^\textnormal{st} + \ln(N/\delta)\big)}\right)\,.
\end{equation}
Hence, even in the least favorable cases, the multitask approach allows to recover the independent baseline by using $b=0$.
Note that this is only made possible by the fact that $\beta_t^\text{small-$b$}$ is tight at $b=0$.

\paragraph{A first bound for small $\epsilon$.}
Here, we instead use the bound
\begin{align*}
\Rmt(T) &\le 4\,\beta^\text{naive}_T(b) \sqrt{\lambda\,T\, \gamma_T^\textnormal{mt}(b)}\\
&= 4\left(B\sqrt{N(1 + b\epsilon^2)} + \lambda^{-1/2}\sqrt{2\big(\gamma_T^\text{mt}(b) + \ln(1/\delta)\big)}\right) \sqrt{\lambda\,T\, \gamma_T^\textnormal{mt}(b)}\\
&\le 4\left(B\sqrt{\frac{(1+b\epsilon^2)(b+N)}{1+b}} + \sqrt{2\left(\gamma_T^\textnormal{st} + \frac{NT}{b} + \ln(1/\delta)\right)}~\right) \sqrt{T\left(\gamma^\text{st}_T + \frac{NT}{b}\right)}\,,
\end{align*}
where we have used the second claim of \Cref{prop:bound_info_gain} and the choice $\lambda = (N+b)/(N+bN)$.
Substituting $b=1/\epsilon^2$ in the above equation, we obtain
\begin{align}
\frac{\Rmt(T)}{6} &\le \left(B\sqrt{\frac{N+1/\epsilon^2}{1+1/\epsilon^2}} + \sqrt{\gamma_T^\textnormal{st} + \epsilon^2NT + \ln(1/\delta)}\right) \sqrt{T\left(\gamma^\text{st}_T + \epsilon^2NT\right)}\nonumber\\
&\le \left(B\big(1 + \epsilon \sqrt{N}\big) + \sqrt{\gamma_T^\textnormal{st} + \epsilon^2NT + \ln(1/\delta)}\right) \sqrt{T\left(\gamma^\text{st}_T + \epsilon^2NT\right)}\nonumber\\
&\le \sqrt{T\gamma_T^\text{st}}\Big(B + \sqrt{\gamma_T^\text{st} + \ln(1/\delta)}\Big) + \epsilon\bigg[B\sqrt{NT\gamma_T^\text{st}} + T\sqrt{N\gamma_T^\text{st}} + B\sqrt{N}T\nonumber\\
&\hspace{5.6cm}  + \epsilon BNT + \epsilon NT^{3/2} + T\sqrt{N\big(\gamma_T^\textnormal{st} + \ln(1/\delta)\big)}\bigg]\nonumber\\
&= \sqrt{T\gamma_t^\text{st}}\Big(B + \sqrt{\gamma_T^\text{st} + \ln(1/\delta)}\Big) + \mathcal{O}\left(\epsilon\,B T \sqrt{N \big(\gamma_T^\text{st} + \ln(1/\delta)\big)} + \epsilon^2 BNT^{3/2}\right)\,.\label{eq:bound_bepsilon}
\end{align}
Bound \eqref{eq:bound_bepsilon} is composed of two terms: the single-task regret, and an additional term which scales with the deviation $\epsilon$ to the average task.
Hence, as $\epsilon$ goes to $0$, i.e., when tasks get more similar, we adaptively recover the single-task bound.
Moreover, note that \eqref{eq:bound_bepsilon} is smaller than \eqref{eq:bound_b0} as long as $\epsilon \le 1/(N^{1/4}\sqrt{T})$.
Hence, by choosing $b = (1/\epsilon^2) \cdot \mathbbm{1}\{\epsilon \le 1/(N^{1/4}\sqrt{T})\}$, we can obtain the minimum of the two bounds.

\paragraph{A second bound for small $\epsilon$.}
Finally, we can use that
\begin{align*}
&\Rmt(T)\\
&~\le 4\,\beta^\text{large-$b$}_T(b) \sqrt{\lambda\,T\, \gamma_T^\textnormal{mt}(b)}\\
&~= 4\left(B \sqrt{\frac{(1+b\epsilon)^2}{1+b} +  \frac{2bN}{1+b} + \frac{2b(1+b\epsilon)^2}{N\lambda^2(1+b)^3}\,T^2} + \lambda^{-1/2} \sqrt{2\big(\gamma_T^\textnormal{mt}(b) + \ln(1/\delta)\big)}\right) \sqrt{\lambda\,T\, \gamma_T^\textnormal{mt}(b)}\\
&~\le 6\left(B\sqrt{\frac{(1+b\epsilon)^2(b+N)}{(1+b)^2N} + \frac{2b(b+N)}{(1+b)^2}\left(1+\frac{(1+b\epsilon)^2}{(b+N)^2}T^2\right)} + \sqrt{\gamma_T^\textnormal{st} + \frac{NT}{b} + \ln(1/\delta)}~\right)\\
&\hspace{1cm}\cdot\sqrt{T\left(\gamma^\text{st}_T + \frac{NT}{b}\right)}\,,
\end{align*}
where we have used the second claim of \Cref{prop:bound_info_gain} and the choice $\lambda = (N+b)/(N+bN)$.
Choosing $b=N/\epsilon^2$, we have
\[
\frac{(1+b\epsilon)^2(b+N)}{(1+b)^2N} = \frac{\left(1+\frac{N}{\epsilon}\right)^2\left(1+\frac{1}{\epsilon^2}\right)}{\left(1+\frac{N}{\epsilon^2}\right)^2} \le \frac{2\left(1 + \frac{N^2}{\epsilon^2}\right)}{\left(1+\frac{N}{\epsilon^2}\right)} \, \frac{1+\epsilon^2}{N+\epsilon^2} \le 2N \frac{5}{N} = 10\,,
\]
\[
\frac{2b(b+N)}{(1+b)^2} \le \frac{2(b+N)}{(1+b)} = 2\,\frac{N + \frac{N}{\epsilon^2}}{1 + \frac{N}{\epsilon^2}} = 2N \frac{1 + \epsilon^2}{N + \epsilon^2} \le 10\,,
\]
\[
\frac{1+b\epsilon}{b+N} = \frac{b\epsilon + N\epsilon - N\epsilon + 1}{b+N} \le \epsilon + \frac{1}{\frac{N}{\epsilon^2} + N} \le \epsilon + \frac{\epsilon^2}{N} \le 3 \epsilon\,.
\]
Substituing in the above bound, we obtain
\begin{align}
\Rmt(T) &= \mathcal{O}\left(B\sqrt{1 + \epsilon^2 T^2} + \sqrt{\gamma_T^\textnormal{st} + \epsilon^2T + \ln(1/\delta)}\right)\sqrt{T\left(\gamma^\text{st}_T + \epsilon^2 T\right)}\nonumber\\
&= \mathcal{O}\Bigg(\sqrt{T \gamma_T^\textnormal{st}}\left(B + \sqrt{\gamma_T^\textnormal{st} + \ln(1/\delta)}\right) + \epsilon B T^{3/2}\sqrt{\gamma_T^\textnormal{st}}\nonumber\\
&\qquad\quad + \epsilon T \bigg(B\sqrt{1 + \epsilon^2 T^2} + \sqrt{\gamma_T^\textnormal{st} + \epsilon^2T + \ln(1/\delta)}\bigg) \Bigg)\nonumber\\
&=\mathcal{O}\left(\sqrt{T \gamma_T^\textnormal{st}}\left(B + \sqrt{\gamma_T^\textnormal{st} + \ln(1/\delta)}\right) + \epsilon B T^{3/2}\sqrt{\gamma_T^\textnormal{st} + \ln(1/\delta)} + \epsilon^2 B T^2\right)\,.\label{eq:bound_bepsilon2}
\end{align}
Again, bound \eqref{eq:bound_bepsilon2} is composed of two terms: the single-task regret, and an additional term which goes to $0$ as $\epsilon$ goes to $0$.
Interestingly, this last part is independent from $N$, which is a consequence of $\beta_t^\text{large-$b$}$ being $\sqrt{N}$ smaller than $\beta_t^\text{naive}$ at $b = +\infty$.
For small values of $T$, namely when $T \le N$, \eqref{eq:bound_bepsilon2} is thus smaller than \eqref{eq:bound_bepsilon}.
Choosing $b = N/\epsilon^2$ when $T \le N$, and as before otherwise, ensures to obtain the minimum of \eqref{eq:bound_b0}, \eqref{eq:bound_bepsilon}, and \eqref{eq:bound_bepsilon2}.
\end{proof}

\subsection{Adapting to unknown tasks' similarity}\label{app:adamtucb}

In~\Cref{alg:adamt-ucb} we summarize the \texttt{AdaMT-UCB} approach discussed in~\Cref{sec:active_learning}. In the misspecification test (Line 8), $\text{lcb}_t^e$ are lower confidence bound functions, defined as:
\begin{equation*}
\text{lcb}^e_t(i,x) = \mu_t\big(i,x \,|\, b^e \big) - \beta_t(b^e) \cdot \sigma_t\big(i,x \,|\, b^e\big)\,,
\end{equation*}
where $b^e$ is the kernel parameter chosen (according to~\Cref{thm:main_thm_regrets}) by each learner $e$. Moreover,
$c$ is an absolute constant such that, by standard concentration arguments and for all times $\tau$,
\begin{equation}\label{eq:concentration_condition}
    \left|U - \sum_{t=1}^\tau\fmt(i_t,x_t) \right| \leq c\sqrt{\tau\ln(\ln(\tau)/\delta)},
\end{equation}
with probability $1-\delta$, see, e.g.,~\citep[Lemma B.1]{Pacchiano2020RegretBB}.

\begin{algorithm}[t]
\caption{\texttt{AdaMT-UCB}}\label{alg:adamt-ucb}
\begin{algorithmic}[1]
\REQUIRE Finite set $\mathcal{E}\subset (0,2]$, learning rate $\eta>0$.
\STATE For $e \in \mathcal{E}$, initialize learner $ \texttt{MT-UCB}({e})$ with $b,\lambda,\{\beta_t\}_{t\in \mathbb{N}}$ according to Theorem~\ref{thm:main_thm_regrets} using $\epsilon=e$.
\STATE Initialize $\tau=U=R=0$, $L^e = 0$, $\forall e \in \mathcal{E}$.
\FOR{t=1,\ldots, }{} 
\STATE Choose learner $e_t = \min_{\epsilon\in \mathcal{E}}$.
\STATE Observe $i_t$ and play action from \texttt{MT-UCB}(${e}_t$), i.e. $x_t  = \arg\max_{x\in \X} \text{ucb}_{t-1}^{e_t}(i_t, x)$.
    \STATE Observe: $y_t = \fmt(i_t, x_t) + \xi_t$, and update \emph{all} learners $\{ \texttt{MT-UCB}({e}) \}_{e \in \mathcal{E}}$ based on observation.
 \STATE Accumulate: 
 \vspace{-.5em}
 \begin{equation*}
    \tau \mathrel{+}= 1, \quad U \mathrel{+}= y_t, \quad R \mathrel{+}= 2\beta_{t-1}^{e_t}\sigma_{t-1}^{e_t}(i_t, x_t), \quad L^e \mathrel{+}= \text{lcb}_{t-1}^e(i_t, x_t), \forall e\in \mathcal{E}.
 \end{equation*}
 \STATE Misspecification test: 
 \vspace{-1em}
 \begin{equation}\label{eq:misspecification_test}
     U + R + c\sqrt{\tau\ln(\ln(\tau)/\delta)} < \max_{e\in \mathcal{E}}L^e
 \end{equation}
 \IF{\text{condition~\eqref{eq:misspecification_test} is true}} \STATE \textcolor{gray}{\# At least one learner is misspecified w.h.p.. Hence, start a new epoch.}
    \STATE $\mathcal{E} = \mathcal{E}\setminus \{e_t\}$ and
    reset $\tau=U=R=0$, $L^e = 0, \forall e\in \mathcal{E}$.
    \ENDIF
\ENDFOR
\end{algorithmic}
\end{algorithm}

\subsubsection{Proof of Theorem \ref{cor:adamt-ucb}}

\thmadamtucb*

Among the set of learners defined by parameters $e\in \mathcal{E}$, we have identified with $e^\star$ the (well-specified) learner with the smallest $e$ such that $e \geq \varepsilon$. Then, our goal is to obtain a regret bound which grows as the regret of learner $e^\star$. We have denoted with $\overline{\Rmt_\star}(T)$ the regret bound of such learner had it been chosen from time $0$.\looseness=-1

We first prove the following auxiliary lemma, which is the analog of~\citep[Theorem~7.1]{Pacchiano2020RegretBB}.
\begin{lem} \label{lem:misspecification_does_not_trigger}
With probability at least $1-\delta$, the misspecification test in \Cref{eq:misspecification_test} does not trigger if all learners in $\mathcal{E}$ are well-specified and their confidence intervals contain $\fmt$.
\end{lem}
\begin{proof}
When all learners $e\in \mathcal{E}$ are well-specified and their intervals contain the true function, for all $t$ it holds $\text{lcb}_t^{e}(i_t, x_t) \leq \max_{x\in \X} \fmt(i_t,x)$. Thus, for each learner $e \in \mathcal{E}$, with probability at least $1-\delta$,
\begin{align*}
L^e &= \sum_{t=1}^{T} \text{lcb}_{t-1}^{e}(i_t, x_t)\\
&\leq \sum_{t=1}^{T} \max_{x\in \X}\fmt(i_t,x)\\
&\leq  \sum_{t=1}^T \fmt(i_t, x_t) +  2\beta_{t-1}^{e_t} \sigma_{t-1}^{e_t}(i_t, x_t) \\ 
& \leq U +  c\sqrt{T\ln(\ln(T)/\delta)} + R\,,
\end{align*}
where, in addition to~\Cref{eq:concentration_condition}, we have used that $\sum_{t} \max_{x\in \X} \fmt(i_t, x) - \fmt(i_t, x_t) \leq \sum_{t} \ucb_t(i_t, x_t) - \text{lcb}_t(i_t, x_t) = \sum_{t} 2 \beta_{t-1}^{e_t}\sigma_{t-1}^{e_t}(i_t, x_t)$. Thus, the misspecification test of \eqref{eq:misspecification_test} does not trigger.
\end{proof}

Let us now bound the overall regret of \texttt{AdaMT-UCB}. First, we can decompose it into the regrets inside each epoch:
\begin{align*}
    \Rmt(T) & = \sum_{t=1}^T \max_{x\in \mathcal{X}} \fmt(i_t, x) - \fmt(i_t, x_t) \\ 
    & = \sum_{s=1}^{\text{$\#$ of Epochs}} \sum_{t\in \text{Epoch-}s} \max_{x\in \mathcal{X}} \fmt(i_t, x) - \fmt(i_t, x_t) \\ 
    & = \sum_{s=1}^{\text{$\#$ of Epochs}}  \Rmt_s(T_s)\,,
\end{align*}
where we have defined $T_s$ to be the duration of epoch $s$ and $\Rmt_s(T_s)$ its corresponding regret. Note that by Lemma~\ref{lem:misspecification_does_not_trigger}, the maximum number of terminated epochs corresponds to the number of misspecified learners in the initial set $\mathcal{E}$. Thus, letting $M$ be such number, with high probability:
\begin{equation}\label{eq:regret_bound_epochs}
    \Rmt(T) \leq \sum_{s=1}^{M+1}  \Rmt_s(T_s)\,.
\end{equation}

\paragraph{During each epoch.}
Let us now look at what happens during each epoch $s$. For simplicity, we will condition on the event that the intervals of learner $e^\star$ contain the true $\fmt$, and on the event of~\Cref{eq:concentration_condition}. Note that by definition, during each epoch the misspecification test has not triggered. In particular, this is true when testing against learner $e^\star$. That is,
\begin{equation*}
    U + R + c\sqrt{T_s\ln(\ln(T_s)/\delta)}\geq \sum_{t\in \text{Epoch-$s$}} \text{lcb}_{t-1}^{e^\star}(i_t, x_t)\,,
\end{equation*}
which, by letting $e_s = \min_{e\in \mathcal{E}}$ be the learner utilized in epoch~$s$, implies:
\begin{equation}\label{eq:misspec_test_not_triggered}
    \sum_{t\in \text{Epoch-$s$}} \fmt(i_t, x_t) +  2\beta_{t-1}^{e_s} \sigma_{t-1}^{e_s}(i_t, x_t)  + 2c\sqrt{T_s\ln(\ln(T_s)/\delta)}\geq \sum_{t\in \text{Epoch-$s$}} \text{lcb}_{t-1}^{e^\star}(i_t, x_t)\,.
\end{equation}


Then, using the above condition:
\begin{align*}
    &\hspace{-3.5cm}\Rmt_s(T_s) - 2c\sqrt{T_s\ln\frac{\ln(T_s)}{\delta}}\\
    &=  \sum_{t\in \text{Epoch-$s$}} \max_{x\in \mathcal{X}} \fmt(i_t, x) - \fmt(i_t, x_t) - 2c\sqrt{T_s\ln\frac{\ln(T_s)}{\delta}} \\
   \text{(By Eq.~\eqref{eq:misspec_test_not_triggered})} \qquad & \leq \sum_{t\in \text{Epoch-$s$}} \max_{x\in \mathcal{X}} \fmt(i_t, x) - \text{lcb}^{e^\star}_{t-1}(i_t,x_t)  + \sum_{t\in \text{Epoch-$s$}} 2\beta_{t-1}^{e_s} \sigma_{t-1}^{e_s}(i_t, x_t)\\ 
    \text{($e^\star$ is well-specified)} \qquad & \leq \sum_{t\in \text{Epoch-$s$}} 2 \beta_{t-1}^{e^\star} \sigma^{e^\star}_{t-1}(i_t,x_t)  + \sum_{t\in \text{Epoch-$s$}} 2\beta_{t-1}^{e_s} \sigma_{t-1}^{e_s}(i_t, x_t) \\ 
   & \leq 4\,\beta_{T_s}(b^{e^\star}) \sqrt{\lambda(b^{e^\star})\,T_s\, \gamma_{T_s}^{\textnormal{mt}}(b^{e^\star})} + 4\,\beta_{T_s}(b^{e_s}) \sqrt{\lambda(b^{e_s})\,T_s\, \gamma_{T_s}^{\textnormal{mt}}(b^{e_s})} \\ 
   & \leq 2 \overline{\Rmt_\star}(T_s)\,.
\end{align*}
In the last two inequalities, we have used the same proof steps of~\Cref{lem:generic_bound} to bound the sum of confidence widths for learners $e^\star$ and $e_s$, and finally, the fact that $e_s \leq e^\star$ and thus the regret bound of learner $e_s$ is bounded by $\overline{\Rmt_\star}(T_s)$ (since the bound from~\Cref{thm:main_thm_regrets} increases with $\epsilon$).

Overall, combining the latter with~\Cref{eq:regret_bound_epochs}, we obtain
\begin{align*}
    \Rmt(T) &\leq \sum_{s=1}^{M+1} 2 \overline{\Rmt_\star}(T_s) + 2c\sqrt{T_s\ln\frac{\ln(T_s)}{\delta}}\\
    & = 2 \sum_{s=1}^{M+1}  C^\star(T_s) \sqrt{T_s} + c\sqrt{T_s\ln\frac{\ln(T_s)}{\delta}} \\
    & \leq 2 C^\star(T) \sqrt{T} \sqrt{M+1} + 2c\sqrt{(M+1)T\ln\frac{\ln(T)}{\delta}} \\ 
    & = 2 \overline{\Rmt_\star}(T) \sqrt{M+1} + 2c\sqrt{(M+1)T\ln\frac{\ln(T)}{\delta}}\,.
\end{align*}
where we have use the fact that \texttt{MT-UCB} regret bounds are of the form $\overline{\Rmt_\star}(T_s) = C^\star(T_s)\sqrt{T_s}$ for some appropriate function $C^\star(T_s)$, see~Theorem~\ref{thm:main_thm_regrets}, and Jensen's inequality. \hfill \qed

\par{\bf How many learners are needed?}
Let $\mathcal{E}$ be the exponential grid $\big\{1, \rho, \rho^2, \ldots, \rho^{M-1}\big\}$, for some $\rho < 1$.
Let $\epsilon \in [0, 2]$ be the true tasks similarity parameter.
By definition, the best learner is better than the learner $m^*$ such that $\rho^{m^*+1} \le \epsilon \le \rho^{m^*}$.
The estimate it uses for $\epsilon$ is better than $\epsilon^* \coloneqq \rho^{m^*}$, which satisfies $\epsilon / \epsilon^* \in [\rho, 1]$.
Hence, the bigger $\rho$, the more precise we are.
However, the number of learners needed also increases with $\rho$.
Indeed, if we want to be able to identify up to $\epsilon_\text{min}$, we have to choose $M$ such that
\[
\rho^{M-1} \le \epsilon_\text{min} \qquad \text{or again} \qquad M \ge 1 + \frac{\log (1/\epsilon_\text{min})}{\log (1/\rho)} \,.
\]

\section{Active Learning}\label{app:active_learning}
\subsection{Proof of Theorem~\ref{thm:active_singlecluster}}

\thmactivelearning*

\begin{proof}
Let $x_\star^i \in \arg\max_{x \in \mathcal{X}}\fmt(i,x)$. Then, the active learning regret of \textnormal{\texttt{MT-AL}} can be bounded as
\begin{align*}
    \Rmt_\textnormal{AL}(T) & := \sum_{t=1}^T \frac{1}{N}\sum_{i=1}^N  \fmt(i, x_\star^i) - \sum_{t=1}^T \frac{1}{N} \sum_{i=1}^N \fmt(i, x_t^i) \\
    & \leq \sum_{t=1}^T \frac{1}{N}\sum_{i=1}^N   \text{ucb}_{t-1}(i, x_\star^i) - \sum_{t=1}^T \frac{1}{N}\sum_{i=1}^N  \text{ucb}_{t-1}(i, x_t^i) + 2 \sum_{t=1}^T \frac{1}{N}\sum_{i=1}^N  \beta_{t-1}^i \sigma_{t-1}(i, x_t^i) \\
    & \leq 2 \sum_{t=1}^T \frac{1}{N}\sum_{i=1}^N  \beta_{t-1}^i \sigma_{t-1}(i, x_t^i) \\ 
    &\leq  2  \sum_{t=1}^T \beta_{t-1}^i \sigma_{t-1}(i_t, x_t^{i_t}) \underbrace{\frac{1}{N}\sum_{i=1}^N 1}_{=1}
\end{align*}
The first inequality holds since by assumption, for all tasks $i$, point $x$, and time $t$, we have $\fmt(i, x) \in [\,\mu_t(i,x) \pm \beta_t^i \cdot \sigma_t(i,x)\,]$. The second one follows since, at each round $t$ \texttt{MT-AL} select $x_t^i = \arg \max_x \text{ucb}_{t-1}(i, x)$ for all $i$, and the third one since $i_t = \arg\max_i \beta_{t-1}^i \sigma_{t-1}(i, x_t^i)$.
\end{proof}

\subsection{Proof of Corollary~\ref{cor:active_learning}}

\corolactivelearning*

\begin{proof}
First, since \textnormal{\texttt{MT-AL}} utilizes the MT regression estimates of Eq.~\eqref{eq:kernel_regression_mean}-\eqref{eq:kernel_regression_var} with parameters set according to Theorem~\ref{thm:main_thm_regrets}, with high probability $\fmt(i, x) \in [\,\mu_t(i,x) \pm \beta_t^i \cdot \sigma_t(i,x)\,]$ and the results of Theorem~\ref{thm:active_singlecluster} holds.
Then, the result follows since, according to the proof of Theorem~\ref{thm:main_thm_regrets}, for every sequence of revealed tasks $\{i_t\}_{t=1}^T$ (in particular the ones chosen by \texttt{MT-AL}), it holds  $2 \sum_{t=1}^T \beta_{t-1}^i \sigma_{t-1}(i_t, x_t^{i_t}) \leq \overline{\Rmt}(T)$, see Appendix~\ref{app:online_learning}. \end{proof}

\subsection{Comparison with \texttt{AE-LSVI}~\texorpdfstring{\cite{li2022near}}{li22}}\label{app:compare_AELSVI} 
The proposed \texttt{MT-AL} algorithm can be compared with the offline contextual Bayesian algorithm~{AE-LSVI}~\cite{li2022near} whose goal is to quickly discover the optimal strategy for each context (task, in our case). Like \texttt{MT-AL}, \texttt{AE-LSVI} selects strategy $x_t^i = \arg \max_x \text{ucb}_{t-1}(i, x)$ for each task $i$. However, unlike \texttt{MT-AL}, \texttt{AE-LSVI} queries the task $i_t = \arg\max_{i\in [N]} [\ucb_t(i,x_t^i) - \max_{x\in \X} \text{lcb}_t(i,x)]$. This, together with a terminal rule for the final reported actions $\{x_T^i\}_{i=1}^N$, ensures the latter are $\mathcal{O}(\beta_T\sqrt{\gamma_T^\text{mt}}/\sqrt{T})$-approximate optimal for each task.
We note that a similar error can also be proven when using \texttt{MT-AL}, by turning the active learning regret bound into a last-iterate approximation error (effectively dividing the regret by $T$). 
However, unlike our approach, it is unclear whether the active learning regret of \texttt{AE-LSVI} is sublinear. In particular, by querying the task with maximal uncertainty \texttt{MT-AL} controls the regrets for all the other tasks (see last inequality in Proof of~\Cref{thm:active_singlecluster}). Instead, the query strategy of \texttt{AE-LSVI} considers a truncated uncertainty, since $[\ucb_t(i,x_t^i) - \max_{x\in \X} \text{lcb}_t(i,x)] \leq [\ucb_t(i,x_t^i) - \text{lcb}_t(i,x_t^i)] = \beta_t^i\sigma_t(i, x_t^i)$.

\section{Additional experimental results}\label{app:additional_exps}

In Figure~\ref{fig:fig_additional_exps} we report additional synthetic experiments for different values of $N$ (number of tasks) and parameter $\delta$ (inversely proportional to the tasks' similarity, see~\Cref{sec:experiments}). We observe that the improvement of MT regression over independent learning increases with the number of tasks and with their similarity (i.e., for small $\delta)$. Moreover, as $N$ increases, the improved confidence intervals further improve over the naive ones, as well as the benefit of active learning compared to uniform sampling. All these considerations conform with our theory.
Experiments were run using 8 CPUs at 3.7 GHz.\looseness=-1

\begin{figure}[t]
        \centering
        \begin{subfigure}[b]{0.49\textwidth}
        \centering \text{\small Online learning}\\
        \vspace{0.1em}
\includegraphics[width=\textwidth]{figs/online_learning/legend.pdf}
        \end{subfigure}
        \begin{subfigure}[b]{0.49\textwidth}
        \centering \text{\small Active learning}\\
        \vspace{0.1em}
        \includegraphics[width=\textwidth]{figs/active_learning/legend.pdf}
        \end{subfigure}\\
        \vspace{1em}
        \hspace{-1em}
        \begin{subfigure}[b]{0.24\textwidth}
        \includegraphics[width=\textwidth]{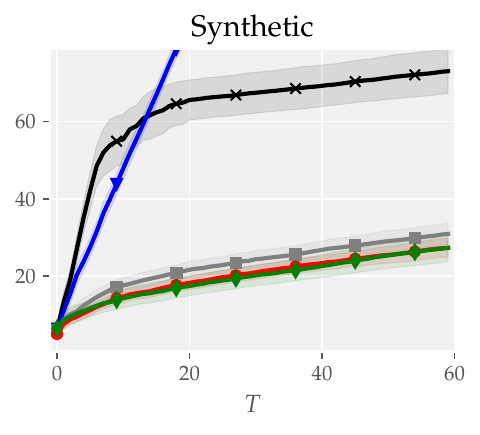}
                \vspace{-1.5em}
        \caption*{$N=3, \delta=0.2$}
        \end{subfigure}
        \begin{subfigure}[b]{0.24\textwidth}
        \includegraphics[width=\textwidth]{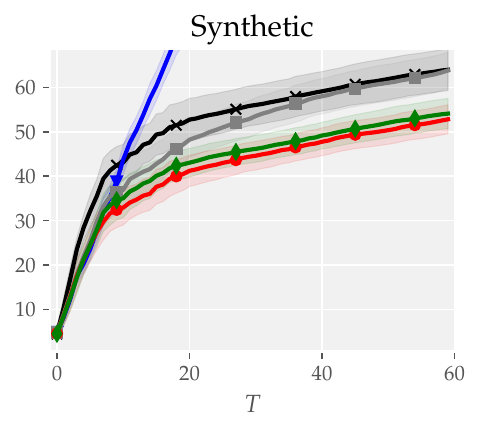}
        \vspace{-1.5em}
        \caption*{$N=3, \delta=0.4$}
        \end{subfigure}
        \hspace{0.5em}
        \begin{subfigure}[b]{0.24\textwidth}
        \includegraphics[width=\textwidth]{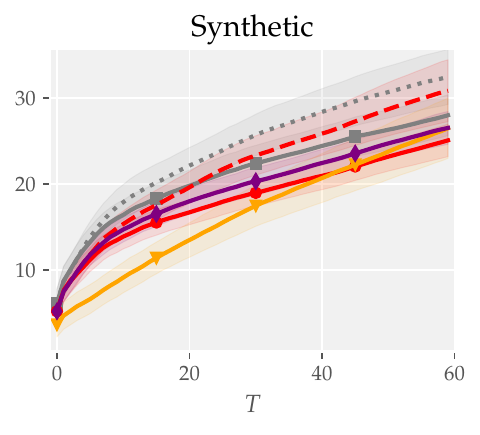}
                \vspace{-1.5em}
        \caption*{$N=3, \delta=0.2$}
        \end{subfigure}
        \begin{subfigure}[b]{0.24\textwidth}
        \includegraphics[width=\textwidth]{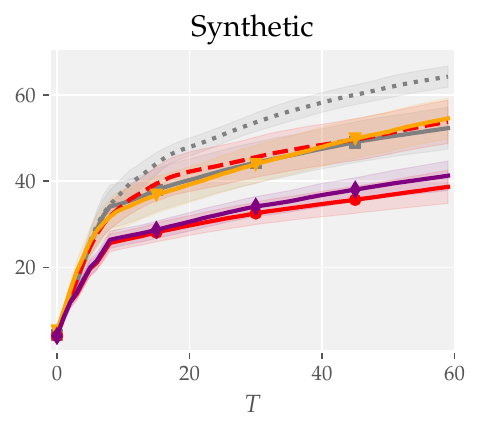}
                \vspace{-1.5em}
        \caption*{$N=3, \delta=0.4$}
        \end{subfigure}
        \\
        \vspace{1em}
       \begin{subfigure}[b]{0.24\textwidth}
        \includegraphics[width=\textwidth]{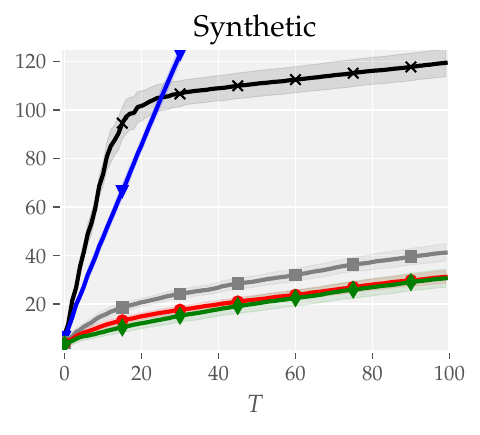}
        \vspace{-1.5em}
        \caption*{$N=5, \delta=0.2$}
        \end{subfigure}
        \begin{subfigure}[b]{0.24\textwidth}
        \includegraphics[width=\textwidth]{figs/online_learning/fig_N_5_d_4_dev_0.4_R_1.0_b_0.05.pdf}
        \vspace{-1.5em}
        \caption*{$N=5, \delta=0.4$}
        \end{subfigure}
        \hspace{0.5em}
        \begin{subfigure}[b]{0.24\textwidth}
        \includegraphics[width=\textwidth]{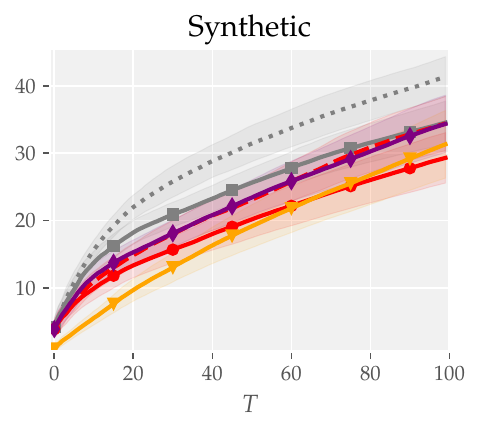}
                \vspace{-1.5em}
        \caption*{$N=5, \delta=0.2$}
        \end{subfigure}
        \begin{subfigure}[b]{0.24\textwidth}
        \includegraphics[width=\textwidth]{figs/active_learning/fig_N_5_d_4_dev_0.4_R_1.0_b_0.05.pdf}
                \vspace{-1.5em}
        \caption*{$N=5, \delta=0.4$}
        \end{subfigure}
       \\
        \vspace{1em}
         \begin{subfigure}[b]{0.24\textwidth}
        \includegraphics[width=\textwidth]{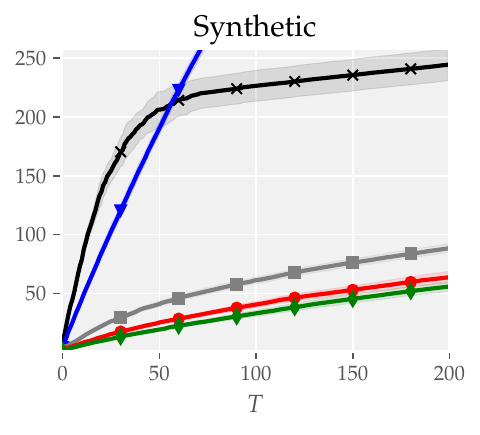}
                \vspace{-1.5em}
        \caption*{$N=10, \delta=0.2$}
        \end{subfigure}
        \begin{subfigure}[b]{0.24\textwidth}
        \includegraphics[width=\textwidth]{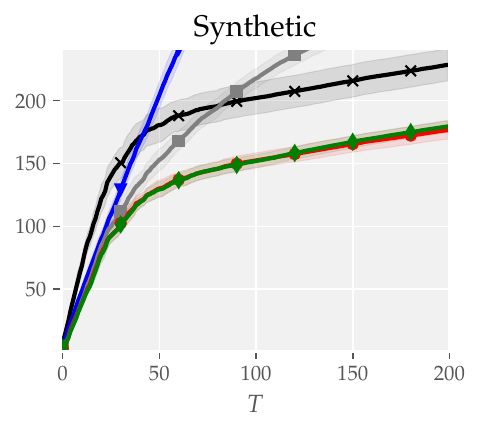}
                \vspace{-1.5em}
        \caption*{$N=10, \delta=0.4$}
        \end{subfigure}
        \begin{subfigure}[b]{0.24\textwidth}
        \includegraphics[width=\textwidth]{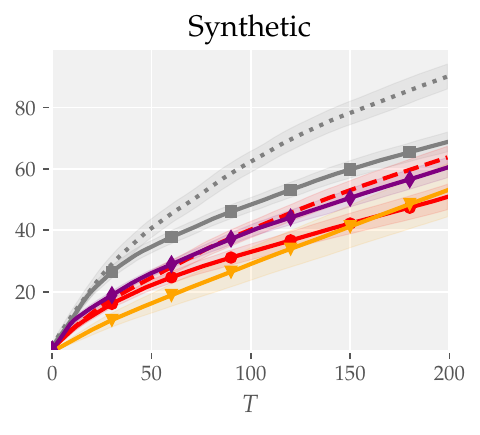}
                \vspace{-1.5em}
        \caption*{$N=10, \delta=0.2$}
        \end{subfigure}
        \hspace{0.5em}
        \begin{subfigure}[b]{0.24\textwidth}
        \includegraphics[width=\textwidth]{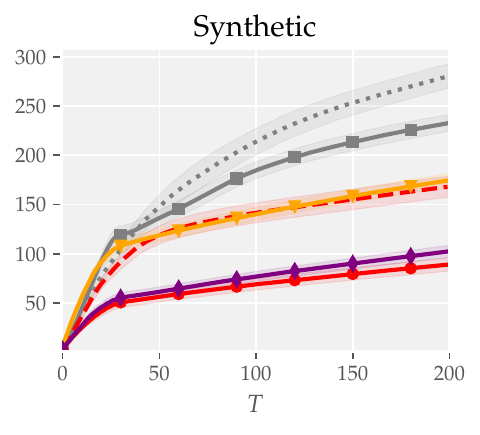}
            \vspace{-1.5em}
        \caption*{$N=10, \delta=0.4$}
        \end{subfigure}
        \caption{Online (left) and active (right) learning regrets of synthetic experiments for different parameters: 
    $N$ is the number of tasks while parameter $\delta$ is inversely proportional to the tasks' similarity, see~\Cref{sec:experiments} for how task vectors are generated.}
    \label{fig:fig_additional_exps}
    \end{figure}

\end{document}